\def\year{2022}\relax
\documentclass[letterpaper]{article} % for initial submission

\usepackage{aaai22}  % DO NOT CHANGE THIS
\usepackage{times}  % DO NOT CHANGE THIS
\usepackage{helvet}  % DO NOT CHANGE THIS
\usepackage{courier}  % DO NOT CHANGE THIS
\usepackage[hyphens]{url}  % DO NOT CHANGE THIS
\usepackage{graphicx} % DO NOT CHANGE THIS
\usepackage{natbib}  % DO NOT CHANGE THIS AND DO NOT ADD ANY OPTIONS TO IT
\usepackage{caption} % DO NOT CHANGE THIS AND DO NOT ADD ANY OPTIONS TO IT
\DeclareCaptionStyle{ruled}{labelfont=normalfont,labelsep=colon,strut=off} % DO NOT CHANGE THIS
\usepackage{newfloat}
\usepackage{listings}
\usepackage{mathtools} % amsmath with fixes and additions
\usepackage{booktabs} % commands to create good-looking tables
\usepackage{tikz} % nice language for creating drawings and diagrams

\usepackage[utf8]{inputenc} % allow utf-8 input
\usepackage[T1]{fontenc}    % use 8-bit T1 fonts
\usepackage{hyperref}       % hyperlinks
\usepackage{url}            % simple URL typesetting
\usepackage{booktabs}       % professional-quality tables
\usepackage{amsfonts}       % blackboard math symbols
\usepackage{nicefrac}       % compact symbols for 1/2, etc.
\usepackage{microtype}      % microtypography
\usepackage{graphicx}
\usepackage{graphics}
\usepackage{subfigure}
\usepackage{wrapfig}
\usepackage{booktabs} % for professional tables
\usepackage{amsfonts}       % blackboard math symbols
\usepackage{amsmath,amsfonts,amssymb,amsthm}
\usepackage{array}
\usepackage{amsthm}
\usepackage{caption}
\usepackage{bbm}
\usepackage{multicol}
\usepackage{lipsum}
\usepackage{xcolor}

\usepackage[ruled,vlined]{algorithm2e}
\usepackage{xr}

\newcommand\ddfrac[2]{{\displaystyle\frac{\displaystyle #1}{\displaystyle #2}}}

\newcommand{\Ebb}{\mathbb{E}}

\newcommand{\Ibf}{\mathbf{I}}
\newcommand{\diff}{\text{d}}

\newcommand{\Ucal}{\mathcal{U}}

\newcommand{\ttrain}{\text{train}}
\newcommand{\ttest}{\text{test}}

\newtheorem{theorem}{Theorem}

\newtheorem{lemma}{Lemma}
\newtheorem{definition}{Definition}

\title{Towards Robust Off-policy Learning for Runtime Uncertainty}
\author {
    Da Xu,\textsuperscript{\rm 1}
    Yuting Ye, \textsuperscript{\rm 2}
    Chuanwei Ruan \textsuperscript{\rm 3}
    Bo Yang \textsuperscript{\rm 4}
}
\affiliations {
    \textsuperscript{\rm 1} Walmart Labs\\
    \textsuperscript{\rm 2} University of California, Berkeley\\
    \textsuperscript{\rm 3} Instacart \\
    \textsuperscript{\rm 4} LinkedIn \\
    daxu5180@gmail.com, yeyt@berkeley.edu, ruanchuanwei@gmail.com, by9ru@virginia.edu
}
\begin{document}

% If your paper is accepted and the title of your paper is very long,
% the style will print as headings an error message. Use the following
% command to supply a shorter title of your paper so that it can be
% used as headings.
%
%\runningtitle{I use this title instead because the last one was very long}

% If your paper is accepted and the number of authors is large, the
% style will print as headings an error message. Use the following
% command to supply a shorter version of the authors names so that
% they can be used as headings (for example, use only the surnames)
%
%\runningauthor{Surname 1, Surname 2, Surname 3, ...., Surname n}

\maketitle

\begin{abstract}
Off-policy learning plays a pivotal role in optimizing and evaluating policies prior to the online deployment. However, during the real-time serving, we observe varieties of interventions and constraints that cause inconsistency between the online and offline settings, which we summarize and term as runtime uncertainty. Such uncertainty cannot be learned from the logged data due to its abnormality and rareness nature. To assert a certain level of robustness, we perturb the off-policy estimators along an adversarial direction in view of the runtime uncertainty. It allows the resulting estimators to be robust not only to observed but also unexpected runtime uncertainties. Leveraging this idea, we bring runtime-uncertainty robustness to three major off-policy learning methods: the inverse propensity score method, reward-model method, and doubly robust method. We theoretically justify the robustness of our methods to runtime uncertainty, and demonstrate their effectiveness using both the simulation and the real-world online experiments.
\end{abstract}

\section{Introduction} \label{sec:intro}
% Contextual bandit gains increasing attention as a key instrument for personalizing decision makings in a broad range of domains, such as online advertising, recommendation, dynamic pricing, medical trails, health care and public policy \citep{canamares2019multi,li2010contextual,chen2013combinatorial,zhu2018robust,athey2017state,jamieson2018bandit,cohen2020feature}. 
The offline optimization and evaluation have been studied intensively in the past few years, as deploying a sub-optimal policy for real-time experiments can be costly and even risky \cite{dudik2011doubly,bottou2013counterfactual,swaminathan2015counterfactual,athey2017efficient}. Nevertheless, it is arguably true that off-policy learning can barely represent real-world scenarios except for a few ideal settings. In particular, real-time policy deployments are inevitably subject to various interventions and constraints that can \emph{not} be unaccounted for in standard off-policy learning. We categorize these unexpected events as \emph{runtime uncertainty} since they are brought about by some anomalous events of the online execution mechanism. Examples include:
\begin{enumerate}
\item in e-commerce recommendation, the product displacement is sometimes subject to real-time business factors such as trending popularity;
\item in personalized online advertisement, during peak hours, the policy execution may be replaced by the backup plan if the response-time agreement can not be satisfied;
\item in autonomous driving, the self-driving vehicle can be hit by accidents, harsh weather, or intrusions that limit the ability to carry out the designed policy.
\end{enumerate}
We conclude that the impacts of runtime uncertainty are often reflected in the changes to the policy's execution. In the above examples, one may expect the following consequences: 1). the original exposure probabilities are adjusted by the instant trends of the items; 2). the exceeding web traffic will be directed to the rollback setting, e.g., non-personalized policy, rather than calling the online inference service of the designed policy; 3). the emergency mechanism, e.g., emergent brake, is triggered and takes over the running policy.
% \begin{enumerate}
%     \item the policy's values are altered after eliminating the out-of-stock items from the action space;
%     \item the exceeding traffic will be directed to the rollback setting, e.g. non-personalized policy, rather than calling the online inference service of the designed policy;
%     \item the emergency-control mechanism (e.g., emergency brake) is triggered and takes over the running policy.
% \end{enumerate}

% To name a few:
% \begin{itemize}
%     \item \textbf{advertisement and recommendation} are often subject to real-time business requirements that are not reflected by the designed policy;
%     \item \textbf{dynamic pricing} can be frequently adjusted when encountering deviations from the normal demand and availability patterns; 
%     \item \textbf{medical trail, health care and public policy} almost always rely on the further professional diagnostic in addition to the given policy.
% \end{itemize}
% In fact, many real-world online learning systems are featured explicitly with a human-in-the-loop decision-making process \cite{chen2018bandit,mcneill2019comparison,holzinger2016interactive}. 
% % Although the feedback data and logging policy faithfully record all scenarios in the past, they are not indicative of such future events. 

In stark contrast to the issue of a non-stationary environment where the underlying reward mechanism changes over time \cite{bubeck2012best}, runtime uncertainty may not alter the reward mechanism. Additionally, runtime uncertainty does not cause distribution change or missing observation of the contextual features, so our problem clearly distinguishes from the setting of distribution-robust and confounding-robust policy optimization (see Section \ref{sec:related} for detail). The major challenge here is that the designed policy cannot be executed as-is due to unexpected events on the executors, which leads to discrepancies between the online and offline settings. It is implausible to remove or characterize these exceptional uncertainties by learning from the logged data, since they are abnormal and intractable by nature. As a consequence, the runtime uncertainty raises a novel problem for off-policy learning, which to the best of our knowledge, has not been studied. 
% Before we proceed, we first distinguish our problem from its counterparts in the related literature. 
% The \textbf{first} one regards the non-stationary environment where the reward mechanism may change \cite{luo2018efficient,bubeck2012best}. 
% % Here, we do not assume a non-stationary reward and only study the practical perturbation on policy during real-world serving. 
% Here, we assume that the potential uncertainty only perturbs the policy during real-world serving, instead of changing the underlying environment.
% The \textbf{second} one relates to the unobserved confounding, where the logging policy is intractable and not all factors are observed, which occurs mostly to the observational studies \cite{kallus2018confounding,liu2013introduction}. For the controlled experiments we study here, the logging policy (including the possible perturbations in the past) are recorded exactly as a part of the system infrastructure (see Section 8 of \cite{slivkins2019introduction} for an example), so we have no sources of confounding. Instead, we work on the problems where unknown factors appear in the online testing stage, which cannot be characterized exactly using only the history data. 

Our strategy is to incorporate a certain level of robustness against potential runtime uncertainty via an offline max-min learning.
We search for an adversarial scenario where runtime uncertainties induce the worst impact on a value function of interest.
% Instead of working on the offline policy prior to deployment, we simply focus on the observed logged data, which is a consequence of the offline policy and runtime uncertainties (see figure \ref{fig:transformation} for illustration) and can benefit the searching for the worst case. We achieve this goal using only the logged feedback and policy: 
Since runtime uncertainties act on policy's execution, we adversarially perturb the logging policy in search for the worst case as if additional runtime uncertainties have played a role. Then we identify the candidate policy with the best offline performance under this worst case. Specifically, for any candidate policy $\pi$, let $\hat{V}(\pi, \pi_u)$ be the estimated reward, e.g., via inverse-propensity score weighting, assuming that the logged feedback was generated under $\pi_u$ ($u$ represents some perturbation mechanism). The key is to introduce an uncertainty set $\Ucal_{\alpha}(\pi_0)$ that surrounds the logging policy $\pi_0$, constructed by such as the $\ell_2$, $\ell_{\infty}$ or Wasserstein balls with radius $\alpha$. To find the optimal policy that is robust to any possible data-generating mechanism in $\Ucal_{\alpha}(\pi_0)$, we formulate adversarial learning objective as:
\[
\underset{\pi}{\text{maximize}}\min_{\pi_u \in \Ucal_{\alpha}(\pi_0)} \hat{V}(\pi, \pi_u),
\]
where $\pi_u$ plays the role of the uncertainty-perturbed policy as if it generated the logged feedback. We can also evaluate the learned policy in such an adversarial scenario. When $\alpha=0$, we revert to the standard off-policy learning where the solution is only optimal under $\pi_0$ and can be sensitive to future online uncertainties. We will further elaborate on the objective and the uncertainty set in Section \ref{sec:bound}.

Based on this framework, we enhance the robustness of the most common offline estimation methods, i.e., the reward-model (RM) method (Q-learning) \citep{jin2018q}, the inverse propensity weighting (IPS) method \citep{horvitz1952generalization} and the doubly robust (DR) method \citep{robins1994estimation}. We provide a practical characterization of the impact of runtime uncertainty. We find it gives the most interpretable and tractable solutions to measure the point-wise deviation from the original design to the final execution, which corresponds to constructing $\Ucal$ using the $\ell_{\infty}$ distance. Then, we study the bounds that reveal how the estimators may fluctuate if the same uncertainty were applied to the logging policy, i.e. the range of $\big[\min_{\pi_u \in \Ucal(\pi_0)} \hat{V}(\pi, \pi_u), \max_{\pi_u \in \Ucal(\pi_0)} \hat{V}(\pi, \pi_u)\big]$ for any given $\pi$. 
%Technically speaking, we only need the lower bound as indicated by the objective function, but we provide a complete characterization since the signs may flip (from reward to regret). Bounding the IPS estimator will be straightforward; however, the reward model will be trickier since it is obtained indirectly via the empirical-risk minimization (ERM). In the sequel, we propose a principled subproblem of the ERM to bound the RM estimator.
The bounds provide a scope of the possible performances, which in turn enables developing an efficient algorithm for the max-min training objective. Furthermore, we study the generalization behavior of the proposed algorithm and rigorously analyze how it leads to robust estimators.

We conduct comprehensive simulation studies to examine the effectiveness of the proposed approach. We also conduct real-world online testings on an e-commerce platform, where our approach compares favorably to standard offline learning. We conclude our contributions as follows.
\begin{itemize}
    \item We study the novel problem of obtaining a runtime-robust policy in the context of off-policy learning.
    \item We design a max-min framework to devise a runtime-robust policy and propose the optimization procedures to bound the estimators in the adversarial setting.
    \item We propose an off-policy algorithm for robust optimization, and theoretically show the tradeoffs and guarantees.
    \item We rigorously examine the effectiveness of our approach via both simulation and real-world experiments.
\end{itemize}

\section{Background and Related work}\label{sec:related}
\begin{table*}[ht]
\centering
\caption{A brief comparison of the runtime uncertainty to other four types of robustness learning.}\label{tbl:uncertainty_comparison}
\resizebox{0.95\linewidth}{!}{
\begin{tabular}{|c|c|p{3cm}|p{6cm}|}
\hline
  & Data distribution &~~~~~~~~~~~~~~Goal& ~~~~~~~~~~~~~~~~~Source of uncertainty \\ \hline
Non-stationary environment                                                         & \begin{tabular}[c]{@{}c@{}}$f_{\ttrain}(Y | X) = f_0(Y | X)$\\ $f_{\ttest}(Y| X) \neq f_0(Y | X)$\end{tabular}                                                                & Adapt $\hat{f}_{\ttrain}$ to $f_{\ttest}$                                                     & $f_{\ttrain}$ and $f_{\ttest}$ only overlap to some extent with perhaps shared common structures.                                                     \\ \hline
Unobserved Confounder                                                      & \begin{tabular}[c]{@{}c@{}}$f_{\ttrain}(X, Y) \neq f_0(X, U, Y)$\\ $f_{\ttest}(X, U, Y) = f_0(X, U, Y)$\end{tabular}                                              & Develop a policy for $f_0(X, U, Y)$                     & Unable to observe some confounders $U$.                                                                                                                 \\ \hline
\begin{tabular}[c]{@{}c@{}}Noisy Label \&\\ Measurement Error\end{tabular} & $f_{\ttrain}(X, Y) \neq f_0(X, Y)$                                                                                                                                                    & Recover the true $f_0$                                                                                             & $X$ or $Y$ might be contaminated.                                                                                                                                                                   \\ \hline
Adversarial Learning                                                       & \begin{tabular}[c]{@{}c@{}}$f_{\ttrain}(X, Y) = f_0(X, Y)$\\ $f_{\ttest}(X, Y) \approx f_0(X, Y)$\end{tabular}                            & Make $\hat{f}_{\ttrain}$ robust to adversarial examples. & $f_{\ttest}$ and $f_{\ttrain}$ are similar, but there are adversarially designed cases that drive the former away from the latter.           \\ \hline
Runtime Uncertainty                                                        & \begin{tabular}[c]{@{}c@{}}$f_{\ttrain}(X, Y) \neq f_0(X, Y)$\\ $f_{\ttest}(X, Y) \neq f_0(X, Y)$\\ $f_{\ttrain}(X, Y) \neq f_{\ttest}(X, Y)$\end{tabular} & Make $\hat{f}_{\ttrain}$ robust to runtime uncertainties & $f_{\ttrain}$ and $f_{\ttest}$ deviate from $f_0$ irregularly, so we aim to develop a policy that is  robust to future uncertainty in $f_{\ttest}$. \\ \hline
\end{tabular}
}
\vspace{-0.3cm}
\end{table*}

Different types of robust machine learning and optimization have been intensively studied in the literature. They provide powerful tools to tackle various problems affected by uncontrolled factors \citep{ben2009robust,bertsimas2011theory}. To better understand the runtime uncertainty, which has not been investigated before, we compare it with other four kinds of robustness, i.e., non-stationary environment, unobserved confounder, noisy label \& measurement error, as well as adversarial learning.

We use $f_0$ to denote a reference distribution of the covariates (or treatment) $X$, outcome (or reward) $Y$, and unknown factors $U$, whose interpretations will be made case-specific. We use $f_{\ttrain}$, $f_{\ttest}$ and $\hat{f}_{\ttrain}$ to denote the training distribution, the testing distribution, and the learned distribution based on the training data, respectively. See Table \ref{tbl:uncertainty_comparison} for the summarized comparisons.

\textbf{Non-stationary environment (mechanism change)}. In the context of off-policy learning, the uncertainties (or known as mechanism change) caused by non-stationary environment have been studied in the previous literature \citep{dudik2014doubly,kallus2018confounding,jagerman2019people,si2020distributional,bareinboim2015bandits}. This field focuses on a problem where the conditional distribution for $Y | X$ under $f_{\ttest}$ differs from that of the training distribution $f_{\ttrain} :=f_0$. It is caused by changes in environments, e.g., customers' interests change over time. 
% In particular, the distributions of the contextual features might differ between training and testing; given the same contextural features, the response in the training distribution is possibly different from the testing. 
To obtain a policy that can work for $f_{\ttest}$, the precondition is that there exists shared structures between $f_{\ttrain}$ and $f_{\ttest}$.  % In particular, \cite{si2020distributional} studies distribution-robust off-policy optimization with respect to the contextual features, which is different from our setting since we do not assume shifts in features' distribution. 

\textbf{Unobserved Confounder}. When there exist unobserved confounders $U$ that affect both the response $Y$ and the treatment $X$, the true underlying distribution $f_0$ is a function of $X$, $U$ and $Y$, but only $f_{\ttrain}(X, Y)$ is observed. The policy developed ignoring these confounders can be sub-optimal and risky to deploy. Research on this topic attempts to develop a robust policy that does not malfunction in the worst case of $f_{\ttest} = f_0$. In particular, the sensitivity analysis is often employed to characterize the robustness of inference outcome when the non-confounding causal assumption is violated to various extents, and confounding-robust techniques are proposed accordingly \citep{rosenbaum1983assessing,rosenbaum2009amplification,ding2016sensitivity,zhao2019sensitivity,kallus2018confounding}. 

\textbf{Noisy Label \& Measurement Error}. The observed response or reward $Y$ may well be contaminated during the collection of large datasets, which is referred to as noisy label learning \citep{natarajan2013learning, northcutt2021confident, zheng2020error, ghosh2017robust}. In addition, the field of measurement error concerns the case where there exist some errors in measuring the covariates $X$, or $X$ is blurred by some systematic noise \citep{neumayer2017robustness, blackwell2017unified, ye2021binomial}. Most studies in this venue focus on recovering the true underlying distribution $f_0(X, Y)$ given the noisy observed training distribution $f_{\ttrain} (X, Y)$. 

\textbf{Adversarial Learning}. Recent years have seen a flurry of studies on the development of algorithms for generating and learning from adversarial examples \citep{goodfellow2014explaining, kurakin2016adversarial,carlini2016towards,madry2017towards}. Such maliciously perturbed examples show little difference from original samples in human perception but can mislead machine learning models to incorrect decisions. In this setting, we still expect $f_{\ttrain}(X, Y) = f_0(X, Y)$, while $f_{\ttest}(X, Y) \approx f_0(X, Y)$ since for the designed $X$ that may fool the algorithm, $f_{\ttest}(X, Y) \neq f_0(X, Y)$ . 

\textbf{Runtime Uncertainty}.
The runtime uncertainty is caused by unexpected events during execution. It drives $f_{\ttrain}$ and $f_{\ttest}$ away from the expected $f_0$ generated by the designed policy. The changes in distributions are irregular and unstructured, which suggests that: 1). $f_{\ttrain} \neq f_0$ and $f_{\ttest} \neq f_0$; 2). the deviations are unsystematic and cannot be learned from data. 
% To tackle the challenges, we innovatively use the logged data to search for a case where the possible runtime uncertainties can deteriorate the value function of interest to the worst case. Then we develop a policy that performs the best in this case. The high-level idea of our max-min learning by perturbing observed data has been seen in distribution-robust and confounding-robust analysis, but
In spite of certain similarities, the properties of the runtime perturbations fundamentally differ from the above scenarios:
\begin{itemize}
    \item the distribution-robust methods only assume the perturbations of the contextual features $X$ \cite{si2020distributional}; the perturbations in the noisy label \& measurement error occur marginally to the response $Y$ or the covariates $X$; the adversarial learning increases the likelihood of seeing the problematic covariates $(X, Y)$ that the machine can make mistakes on. On the other hand, the runtime perturbations allow the entire policy to be perturbed;
    \item by the definition of confounders, most confounding-robust methods assumes perturbations to be independent of the actions that were taken, e.g., $\pi(a|x,u) \propto \pi_0(a|x) + \lambda u$ \cite{kallus2018confounding}; however, the nature of runtime uncertainty is much more complex since its effect can be dependent on the actions, so the perturbation should be able to adapt to any given policy.
\end{itemize}

% We adopt a similar idea for our problem, but with different implications and properties. 
% To achieve robustness in offline learning, we design a minimax training objective that is common to the robust machine learning \citep{sinha2017certifying,kurakin2016adversarial,lanckriet2002robust}.
Another line of research studies the estimator instability caused by the variance issues of specific off-policy estimators \citep{swaminathan2015self,ma2019imitation,farajtabar2018more,xie2019towards,thomas2016data,vlassis2019design}. Their methods for improving the estimators' stability can be adapted to our solutions.

\section{Preliminary} \label{sec:prelim}
In this section, we introduce the notations, basic concepts, problem setup as well as the off-policy estimators of interest. 
% \qye{I suggest using $\pi_0$ for offline policy prior to deployment, $\pi_d$ for deployment policy that encounters runtime uncertainty. People are used to the subscript $0$ as some ``truth''. I feel such change can reduce confusion to some extent.}

\textbf{Notation}. Let $\mathcal{X}$ be the context (feature) space, $\mathcal{A}=\{1,\ldots,k\}$ be the action space, and $r(a,x)$ be the fixed \emph{value} (reward or regret) that is revealed under action $a \in \mathcal{A}$ and context $x\in\mathcal{X}$. The challenge of offline learning and evaluation is largely due to the partial-observation of the complete reward --- we do not know the rewards of untaken actions. 
For an individual with context $x^{(i)}$ and provided with action $a^{(i)} \in \mathcal{A}$ at the $i^{\text{th}}$ round, the feedback data collected after $T$ rounds is given by the set of triplets $h_T=\big\{(x^{(i)}, a^{(i)}, r(a^{(i)},x^{(i)}))\big\}_{i=1}^T$. 
The \emph{logging policy}, which gives the conditional probability of an action under particular context and history, is denoted by: $\pi_0\big(a | x^{(T+1)}, h_T\big)$.
Since our primary focus is runtime uncertainty, we assume the \emph{logging} policy is \emph{stationary} and the contexts are static such that $\pi_0\big(a | x^{(T+1)}, h_T\big) = \pi_0\big(a|x^{(T+1)}\big)$. 

\textbf{Value function}. For off-policy learning, given the context $x$ and the complete reward $r(a,x)$ for all actions, the \emph{value} of a policy is: $V(\pi) = \mathbb{E}_{\pi(a|x)}[r(a,x)]$. Policy evaluation estimates the value $\hat{V}(\pi)$ using the collected feedback data $\{(x_i, a_i, r_i)\}_{i=1}^n$ and the logging policy $\pi_0(a_i|x_i)$ if available. We aim at \emph{reward-maximization} unless otherwise specified. Policy optimization address the problem of finding the optimal policy from a parametric family of candidate policies $\mathcal{F}$, for instance, $\pi^* = \arg \max _{\pi \in \mathcal{F}} \hat{V}(\pi)$.

Estimating the value of an alternative policy $\pi$ is challenging because we do not observe the \emph{potential value} (defined as below) for the actions that are not selected.

\begin{definition}
\emph{
A \emph{potential value} $R(a,x)$ is the \emph{value} that would have been observed if the individual with context $x$ received action $a$. 
}
\end{definition}

By definition, the reward $r(a,x)$ is an averaged version of $R(a,x)$ that integrates out all the possible actions under $\pi_0$:
\begin{equation} 
\label{eqn:potential-value}
\begin{split}
 r(a, x) &:= \Ebb_{\pi_0} \big[R(a,x) \big| X = x\big] \\ 
 & = \sum_{\tilde{a} \in \mathcal{A}} \Ebb\big[R(a,x) \,\big|\, A = \tilde{a}, X = x\big] \pi_0(\tilde{a}|x),
\end{split}
\end{equation}
which takes the form of the \emph{averaged potential value (APV)}. In view of Pearl’s framework \cite{pearl2018book}, $R(a, x)$ is a rung-3 quantity and $r(a,x)$ is a rung-2 quantity that averages out the counterfactual effect. This viewpoint has a profound implication for offline evaluation. For example, an item is tagged with a price $a$ while the true price suggested by the market is $\tilde{a}$. When $\tilde{a}$ is larger than $a$, the quantity $\Ebb \big [ R(a, x) \, \big | \, A = \tilde{a}, X = x\big ]$ reveals how the customer may react to the lowered price. In theory, averaging all possible actions $\tilde{a}$ forms a comprehensive and systematic evaluation for any specified action $a$ and context $x$. 

\textbf{Offline estimators}. It is unrealistic to observe $(r(a), \tilde{a}, x_i)$ for $a \neq \tilde{a}$ due to the partial-observation nature. A huge body of literature has been devoted to estimating the averaged potential value $\Ebb_{\pi_0} \big[R(a,x) \big| X = x\big]$, such as the potential outcome models \citep{robins1994estimation}, structural models \citep{robins2000marginal,pearl2009causality}, two-stage regression models \citep{angrist1995two}. Broadly speaking, they attempt to model the reward mechanism, and is also referred to as the Q-learning in reinforcement learning \citep{sutton2018reinforcement}. Without loss of generality, we refer to them as the \emph{reward-model (RM)} method denoted by $\hat{V}_{\text{RM}}$.
% The dependency of $\hat{V}_{\text{RM}}$ on the original policy $\pi_0$ is implicit, where the joint distribution of $(a,x)$ is decided by $\pi_0$.
\begin{equation}
\label{eqn:reg}
 \hat{V}_{\text{RM}}(\pi;\pi_0) = \frac{1}{n} \sum_{i=1}^n \sum_{a \in \mathcal{A}} \pi(a|x_i) \hat{r}(a , x_i),
\end{equation}
where $hat{r}$ is an estimate of Eq. \eqref{eqn:potential-value}. RM is often subject to model misspecifications that lead to biased estimation. Using the idea of importance sampling, the \emph{inverse-propensity score method} (IPS) described below can estimate the policy value unbiasedly \citep{horvitz1952generalization}. By correcting for the shift in action probability between $\pi_0$ and $\pi$, IPS is less prone to the bias issues.
\begin{equation}
\label{eqn:ips}
 \hat{V}_{\text{IPS}}(\pi;\pi_0) = \frac{1}{n} \sum_{i=1}^n  \frac{\pi(a_i|x_i)}{\pi_0(a_i|x_i)} r_i.
\end{equation}

In practice, the variance-stabilized versions of IPS, e.g., the normalized and truncated IPS, are considered more often \citep{vlassis2019design,gilotte2018offline}. We defer the discussion to the appendix to avoid unnecessary repetitions.

The \emph{doubly robust estimator} (DR) is a popular \emph{control variate} method that effectively leverages both RM and IPS \citep{robins1994estimation}. The notion comes from the fact that DR gives consistent estimation when either RM or IPS is consistent. Here, $\hat{V}_{\text{DR}}$ is given by:
\begin{equation*}
  \label{eqn:dr}
\begin{split}
    \hat{V}_{\text{DR}}(\pi;\pi_0) = & \frac{1}{n}\sum_{i=1}^n \big \{ \hat{r}(x_i) +  \frac{\pi(a_i|x_i)}{\pi_0(a_i|x_i)} \big(r_i - \hat{r}(a_i,x_i)\big) \big \},
\end{split}
\end{equation*}
where $\hat{r}(x_i)$ is the RM estimation: $\sum_{a\in \mathcal{A}} \pi(a|x_i) \hat{r}(a,x_i)$.

\section{Offline Estimators under $\ell_{\infty}$ Uncertainty} \label{sec:bound}

Unlike other domains where $\Ucal_{\alpha}$ can have many options, finding the right uncertainty set is critical to the interpretation and effectiveness of our approach. The reasons are:
\begin{enumerate}
    \item The formulation of $\Ucal_{\alpha}$ should conform to how online uncertainty arises in practice, which in turn explains the type of robustness we are asserting. For instance, the $\ell_2$ ball would suffice if we believe the impact of online uncertainty is close to average on each action.
    \item The constraint optimization of $\min_{\pi_u \in \Ucal_{\alpha}(\pi_0)} \hat{V}(\pi, \pi_u)$ can be intractable or extremely challenging to solve; for instance, when $\Ucal_{\alpha}(\pi_0)$ is given by the Wasserstein ball\footnote{ERM under Wasserstein's constraint are often converted to another min-max optimization \cite{lee2018minimax}; which means our objective will become a max-min-max problem.}.  
\end{enumerate}

In our study, we find the $\ell_{\infty}$ distance gives a nice tradeoff in terms of the \emph{interpretability}, \emph{practicality}, and the \emph{robustness guarantee}. We first discuss the interpretation of using $\ell_{\infty}$ distance. Recall from Section \ref{sec:intro} that runtime uncertainty causes unexpected interventions and constraints to the policy's execution, which can impact $\pi(a|x)$ for \emph{any} $a$ and $x$. It means runtime uncertainty can cause deviation that is best measured in a point-wise fashion, e.g. $\max_{a} \big|\pi_u(a|x) - \pi_0(a|x)\big|$ for a given $x$, rather than by some average, e.g. $\frac{1}{k}\sum_{a} \big|\pi_u(a|x) - \pi_0(a|x)\big|$. To better leverage the fraction format in IPS and DR, we 
replace the absolute difference by the ratio of: $\max\{\frac{\pi_{u}(a | x)}{\pi_{0}(a | x)}, \frac{\pi_{0}(a | x)}{\pi_{u}(a | x)} \}$. If we treat the policies $\pi_0$ and $\pi_u$ as vectors, by the same essence, their deviation is exactly measured by the $\ell_{\infty}$ norm. Therefore, the uncertainty set $\Ucal_{\alpha}(\pi_0)$ follows:
\begin{equation}
\label{eqn:perturb-def}
    \Big\{\pi_u: \max_{a\in\mathcal{A},x\in\mathcal{X}} \max \Big\{\frac{\pi_{u}(a | x)}{\pi_{0}(a | x)}, \frac{\pi_{0}(a | x)}{\pi_{u}(a | x)} \Big\}\leq e^{\alpha} \Big\}.
\end{equation}
Here, the adversarialness endowed in $\pi_u$ can depend on the actions and contexts given the logged data, e.g. $\pi_u(a|x) \propto \pi_0(a|x) + u(a,x)$ where $u$ is some random function. It means the perturbation can be made \emph{policy-specific}, as opposed to policy-agnostic case where the perturbation is uniform for all the actions (such as in confounding-robust optimization). It more closely resembles the mechanism of real-world online uncertainty.
For the sake of notation, we also denote the constraints by the short hand: $e^{-\alpha}\leq \pi_u / \pi_0 \leq e^{\alpha}$. 
We now establish the minimax objective for robust off-policy learning under the $\ell_{\infty}$ uncertainty set:
\begin{equation}
\label{eqn:robust-def1}
    \underset{\pi}{\text{maximize}} \min_{\pi_u:\, e^{-\alpha} \leq \pi_0/\pi_{u} \leq e^{\alpha}} \hat{V}(\pi; \pi_u).
\end{equation}

Note that the candidate policy $\pi$ is not involved in the constraints of the minimization step. Therefore, if we have a subroutine that efficiently computes for any candidate $\pi_{\theta}$:
\begin{equation}\label{eqn:min_val}
\tilde{V}(\pi_{\theta}) := \min_{} \hat{V}(\pi_{\theta}; \pi_u) \text{ s.t. } e^{-\alpha} \leq \pi_0/\pi_{u} \leq e^{\alpha}, 
\end{equation}
or provides the lower bound, i.e. $\underline{\hat{V}(\pi_{\theta})} \leq \tilde{V}(\pi_{\theta})$, we can divide and conquer the minimax objective. This technique resembles the well-known \emph{expectation-maximization} and \emph{minorize-maximization} algorithms \citep{lange2016mm,dempster1977maximum}, whose implications are discussed in the appendix. In the sequel, we focus on deriving the bounds $\underline{\hat{V}(\pi_{\theta})}$ for the three off-policy estimators of interest.

% \subsection{The Reward-model Method}
% \label{sec:reg}
\textbf{The Reward-model Estimator}.
To derive the constrained lower bound of $\hat{V}_{RM}$ defined in \eqref{eqn:reg}, it amounts to bounding the APV of \eqref{eqn:potential-value} such that $\pi_u$ satisfies the constraint \eqref{eqn:perturb-def}. We %leverage the Radon-Nikotyn derivative and a strong-duality argument to 
convert the constraint optimization into a subproblem of the standard ERM, as we show in Lemma \ref{lemma:apv-eqv}. The technical details and proof are deferred to the appendix.

\begin{lemma}
\label{lemma:apv-eqv}
When $\pi_u$ satisfies \eqref{eqn:perturb-def}, then
\begin{eqnarray*}
  &&\Ebb_{\pi_u} \big [R(a, x)|A = a', X = x\big ]\nonumber\\
  &\geq& \min_{f_{a, a'}(\cdot)} \Ebb \big[\ell _{\alpha} \big(R(a,x) \, , \, f_{a, a'}(x)\big) \,\big|\, A=a, X=x \big],\nonumber \label{eqn:ERM}
\end{eqnarray*}
where the loss function $\ell_{\alpha}$ is specified by:
\begin{equation}
\label{eqn:auxiliary-loss}
\begin{split}
    & \ell _{\alpha} \big(R(a,x) \, , \, f_{a, a'}(x))\big) = \big\{ R(a,x)  - f_{a, a'}(x) \big\}^2_{+} \\
    & \qquad + e^{2\alpha}  \big\{ R(a,x) - f_{a, a'}(x) \big\}^2_{-}.
\end{split}
\end{equation}
Here, $\{\cdot\}_{+}$ and $\{\cdot\}_{-}$ are respectively the positive and negative parts. Since the potential value $R(a,x)$ is observed given $A=a$, the above setting describes a subproblem of the ERM.
\end{lemma}
The result in Lemma \ref{lemma:apv-eqv} holds for all the alternative $a' \in \mathcal{A}$ with $a'\neq a$, so  we only need to compute $f_a := f_{a,a'}$ under any $a'$. The lower bound $\underline{\hat{V}_{\text{RM}}(\pi_{\theta})}$ %for $\Ebb_{\pi_u} \big[R(a,x) \big| X = x\big]$
is easily obtained with:
\begin{equation*}
\label{eqn:bound-reg}
\begin{split}
 & \sum_{a' \in \mathcal{A}} \Ebb\big[R(a,x) \,\big|\, A = a', X = x\big] \pi_u(a'|x) \\
 & \geq \min_{\pi_u: e^{-\alpha} \leq \pi_0/\pi_{u} \leq e^{\alpha}} \big\{\big(1- \pi_u(a|x)\big) \hat{f}_{a}(x) + \\
 & \qquad \qquad \qquad \qquad \pi_u(a|x) \Ebb \big[R(a,x) \,\big|\, A=a,x\big] \big\}
\end{split}
\end{equation*}
where $\hat{f}_a$ is the solution to the auxiliary ERM problem in (\ref{eqn:ERM}), and $\Ebb \big[R(a,x) \,\big|\, A=a,x\big]$ can be estimated as in the standard off-policy setting. Therefore, computing the lower bound for RM involves solving a standard and a subproblem of the ERM, which is computationally efficient.
% Finally, for the optimization problem in the last line, we can easily obtain the explicit solution for the minimal value.
\textbf{Upper bound}: we point out that lower-bounding the DR estimator also requires the upper bound of RM. Using the same arguments, the upper bound can be obtained using a similar auxiliary ERM approach, by replacing the $e^{2\alpha}$ in (\ref{eqn:auxiliary-loss}) with $e^{-2\alpha}$. We denote the upper bounds by $\big\{\hat{g}_a\big\}_{a \in \mathcal{A}}$. 

% \subsection{The Inverse Propensity Scoring Method}
% \label{sec:IPS}
\textbf{The IPS Estimator}.
Lower-bounding the IPS estimator is more straightforward according to our design. The lower-bounding objective in (\ref{eqn:min_val}) directly becomes:
% \begin{equation}
% \label{eqn:optimization-IPS}
% \begin{split}
%     & \qquad \qquad \text{minimize} \quad \frac{1}{n} \sum_{i=1}^n  \frac{\pi(a_i|x_i)}{p(a_i | x_i)} r_i \\
%     & \text{s.t.} \quad e^{-\alpha} \pi_{0}(a_i | x_i) \leq p(a_i | x_i) \leq e^{\alpha} \pi_{0}(a_i | x_i), \\
%     & \quad \quad  \sum_{a \in \mathcal{A}} p(a|x_i) \leq 1, \text{ for } i=1,\ldots,n.
% \end{split}
% \end{equation}
\begin{equation*}
\label{eqn:optimization-IPS}
\begin{array}{rl}
    \text{minimize} & \frac{1}{n} \sum_{i=1}^n  \ddfrac{\pi(a_i|x_i)}{p(a_i | x_i)} r_i \\[10pt]
    \text{s.t.} & e^{-\alpha} \pi_{0}(a_i | x_i) \leq p(a_i | x_i) \leq e^{\alpha} \pi_{0}(a_i | x_i), \\[5pt]
    & \sum_{a \in \mathcal{A}} p(a|x_i) = 1, \text{ for } i\in[n],
\end{array}
\end{equation*}
where we use $p(\cdot|\cdot)$ to denote the optimization variables.
The second set of constraints is necessary because it makes sure that the solution constitutes a valid policy. 
The optimization problem for IPS can be solved explicitly using a change of variable: $w(a_i,x_i) := 1/p(a_i|x_i)$. It follows that both the objective and constraints are convex in $w$. In practice, the IPS estimator can suffer from the variance issues, so the variants such as the normalized and truncated IPS are often used instead in practice \citep{vlassis2019design,gilotte2018offline}. 
We defer the discussions for the variant methods to the appendix to avoid unnecessary repetitions. 

% \subsection{The Doubly-robust Method}\label{sec:DR}

\textbf{The DR Estimator}. 
Lower-bounding DR is a straightforward combination of what we have shown for RM and IPS:
\begin{equation*}
\label{eqn:optimization-DR}
\begin{split}
& \text{minimize}\quad  \hat{V}_{\text{DR}}(\pi_{\theta};p,r) :=    \frac{1}{n}\sum_{i=1}^n \Big \{ \sum_{a\in \mathcal{A}} \pi(a|x_i) r(a_i,x_i) \\ 
    & \qquad \qquad \qquad \qquad + \frac{\pi(a_i|x_i)}{p(a_i|x_i)} \big(r_i - r(a_i,x_i)\Big \} \\
& \qquad \text{s.t. } r(a_i,x_i) \in \mathcal{R}(\pi_0, \alpha) \text{ and } p(a_i,x_i) \in \Pi(\pi_0, \alpha),
\end{split}
\end{equation*}
with the constraint sets given by:
% \begin{equation*}
% \begin{array}{l}
%     \hat{f}_{a_i}(x) \leq r(a_i,x_i) \leq \hat{g}_{a_i}(x), \\ 
%     e^{-\alpha} \pi_{0}(a_i | x_i) \leq p(a_i | x_i) \leq e^{\alpha} \pi_{0}(a_i | x_i), \\
%     \sum_{a \in \mathcal{A}} p(a|x_i) \leq 1, \text{ for } i=1,\ldots,n.
% \end{array}
% \end{equation*}
\begin{equation*}
\begin{split}
  & \mathcal{R}(\pi_0, \alpha) := \big\{r(a_i, x_i): \hat{f}_{a_i}(x_i) \leq r(a_i,x_i) \leq \hat{g}_{a_i}(x_i); i \in [n]\big\}, \\
  & \Pi(\pi_0, \alpha) := \big\{p(a_i| x_i): e^{-\alpha}  \leq \frac{p(a_i | x_i)}{\pi_{0}(a_i | x_i)} \leq e^{\alpha}; \\
  &\qquad \qquad \qquad \sum_{a \in \mathcal{A}} p(a|x_i) = 1; i\in [n] \big\}.
\end{split}
\end{equation*}
The first set of constraints use the RM bounds to characterize the uncertainty of the reward models, and both the $\hat{f}_a$ and $\hat{g}_a$ can be computed beforehand. Although the objective for DR is not jointly convex in $r(a,x)$ and $w(a,x) := 1/p(a|x)$, it is coordinate-wise convex (affine). Therefore, we can employ any off-the-shelf solver. Also, the objectives for the IPS and DR are separable for each $x_i$, so we can efficiently parallelize the computations for each observation.

%=======================================================

%=======================================================

\section{Learning Algorithm and Guarantee}

The learning objective can be nonconcave-nonconvex in general, so we cannot switch the $\min$ and $\max$. Therefore, the lower-bounding methods from the previous section, which holds for \emph{any} given $\pi_{\theta}$, play an essential role in finding the approximate solution.
We illustrate how to plug in the lower bounds for DR as an example since it includes both the IPS and RM estimator as special cases (see Algorithm \ref{algo:algo}). 

From the learning-theoretical perspective, it is important to understand how the proposed approach affects the generalization performance while asserting robustness against runtime uncertainty.
In the following theorem, we characterize the generalization of policy improvement of the max-min solution for DR, given any $\alpha>0$.

\begin{theorem}
\label{thm:main}
Suppose that for all $\alpha>0$, there exists a constant $M_{\alpha}$ such that $\max_{a,x} |\hat{f}_a(x)| \leq M_{\alpha}$ and $\max_{a,x} |\hat{g}_a(x)| \leq M_{\alpha}$. Also, we assume $\pi_0(a|x_i) \in (q, 1-q)$ for some $q>0$. Let $\bar{r} := \max_{i} |r_i|$, then for all $\pi_{\theta} \in \mathcal{F}$ and $\forall \delta > 0$, with probability as least $1-\delta$:
\begin{equation*}
\label{eqn:main}
\begin{split}
     & V(\pi_{\theta}) - V(\pi_0) \geq  \underbrace{\min_{p\in \Pi, r \in \mathcal{R}} \hat{V}_{\text{DR}}(\pi_{\theta};p, r) - \hat{V}(\pi_0)}_{\text{I}}  \\
     & \underbrace{- 6\big( \frac{q+1}{q}M_{\alpha} + \bar{r} \big)\sqrt{\frac{2\log \frac{3}{\delta}}{n}}   - 2 \max \big\{M_{\alpha}, \frac{\bar{r}}{q} \big\}\mathcal{R}_n(\mathcal{F})}_{\text{II}},
\end{split}
\end{equation*}
where $\hat{V}(\pi_0)$ is the logging policy's value on the training data, and $V(\pi)=\Ebb \hat{V}(\pi)$ is defined as in Section \ref{sec:prelim}. 
\end{theorem}

The proof is relegated to the appendix. The significance of Theorem \ref{thm:main} is to reveal the two critical components that control generalization: \emph{I}. the empirical policy improvement under the proposed minorize-maximization algorithm; \emph{II}. the composite terms of the policy complexity and the degree of uncertainty reflected via $M_{\alpha}$ (see Appendix A for detail). 

In particular, by enriching $\mathcal{F}$, we are more likely to make \emph{I} positive on training data, but we then suffer from a larger negativity of \emph{II}. This tradeoff is consistent with the standard generalization for supervised learning. More importantly, the magnitude of \emph{II} also increases with $M_{\alpha}$. Notice that $M_{\alpha}$ is often non-decreasing in $\alpha$ as the RM bounds get looser. As a consequence, having $\alpha>0$ further penalizes the model complexity and the slack term in \emph{II}, so increasing $\alpha$ will encourage the algorithm to select the policy that achieves less empirical improvement but has smaller complexity. It explains from the theoretical perspective how our approach can lead to a policy that performs better under runtime uncertainty. To summarize, the result in Theorem \ref{thm:main} shows rigorously how and why introducing $\pi_u$ with $e^{-\alpha} \leq \pi_0/\pi_{U} \leq e^{\alpha}$ can improve the robustness of the learned policy.

\begin{algorithm}[tbh]
\SetKwInOut{Input}{Input}
\SetKwInOut{Output}{Output}
\SetAlgoLined
\Input{The uncertainty level $\alpha$, history logging policy $\pi_0$, feedback data.}
% \Output{The runtime-robust policy $\pi_{\theta^{\text{new}}}$.}
 \textbf{Initialize} $\theta^{\text{new}}$, $p^*$, $r^*$; \\
 Compute the constraints sets $\mathcal{R}(\pi_0;\alpha)$ for RM by solving the ERM subproblem as in Lemma \ref{lemma:apv-eqv}; \\
 Compute the constraints sets $\Pi(\pi_0;\alpha)$; \\
 \While{$\min_{p,r} \hat{V}_{\text{DR}}(\pi_{\theta^{\text{new}}}) \geq \min_{p,r} \hat{V}_{\text{DR}}(\pi_{\theta^{\text{old}}})$}{
  Let $\theta^{\text{old}} = \theta^{\text{new}}$, compute $\theta^{\text{new}} = \arg \max_{\theta} \hat{V}_{\text{DR}}(\pi_{\theta};p^*, r^*)$ using suitable optimization method; \\
  Solve: $p^*, r^* = \arg \min_{p\in \Pi, r \in \mathcal{R}} \hat{V}_{\text{DR}}(\pi_{\theta^{\text{new}}};p, r)$. \\
 }
 \caption{Robust Off-policy Learning with DR}
 \vspace{-0.1cm}
 \label{algo:algo}
\end{algorithm}

% \subsection{Summary}
% In this section, we provide efficient methods to lower-bound the off-policy estimators: $\underline{\hat{V}_{\text{RM}}(\pi_{\theta})}$, $\underline{\hat{V}_{\text{IPS}}(\pi_{\theta})}$ and $\underline{\hat{V}_{\text{DR}}(\pi_{\theta})}$, after converting the runtime uncertainty to the offline setting via introducing $\pi_u$. The proposed methods are principled since they follow directly from the original estimators and the definition of $\pi_u$.

% \begin{figure}[htb]
%     \centering
%     \includegraphics[width=0.7\linewidth]{images/optimization.png}
%     \caption{\footnotesize A sketched visual illustration for the minorize-maximization optimization procedure.}
%     \label{fig:optimization}
% \end{figure}

% \section{Minorize-maximization Policy Optimization}\label{sec:optimization}
% \input{optimization}

\section{Experiment and Result} \label{sec:experiment}
\begin{figure*}[hbt]
    \centering
    \includegraphics[width=0.84\linewidth]{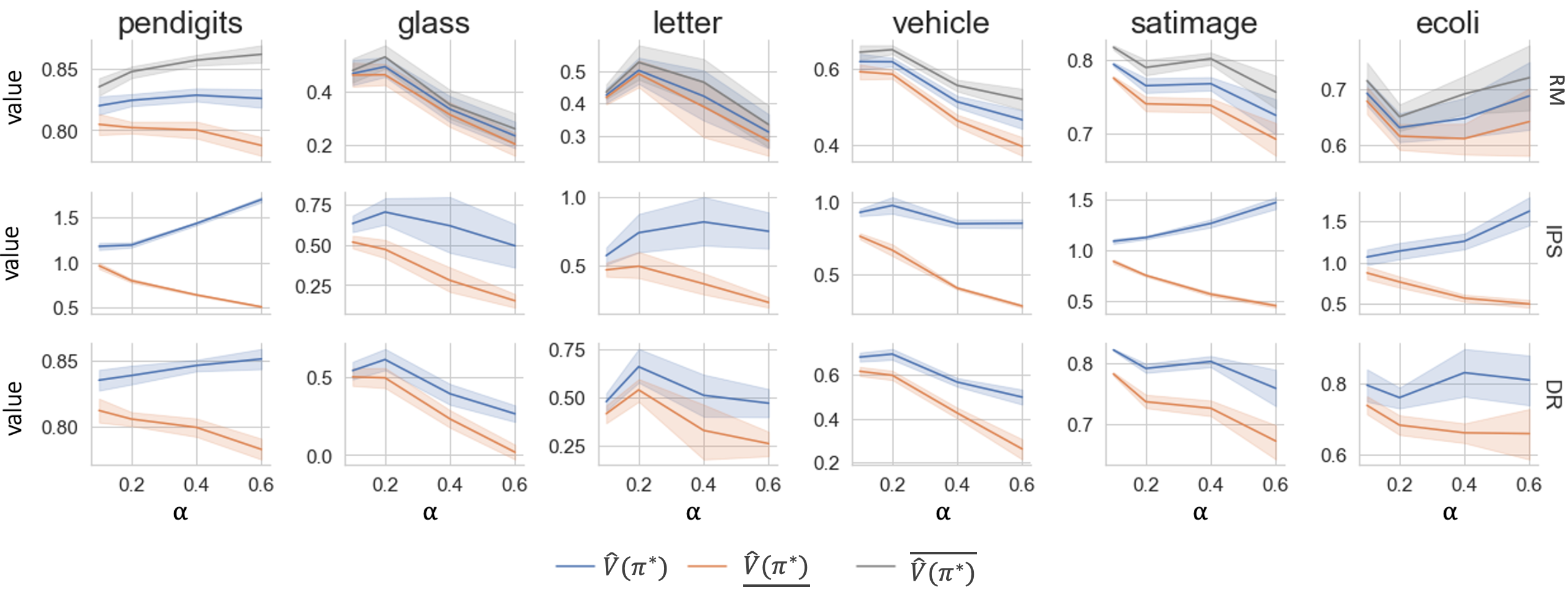}
    \caption{\footnotesize The bounds for the RM, IPS and DR estimators for the uncertainty-perturbed data, under different values of $\alpha$. The values for the original estimators are in {\color{blue}{blue}}, and the values of the bounding methods are given by {\color{orange}{orange}} and {\color{gray}{gray}}. We only show the lower bounds for IPS and DR for the sake of presentation as their tail gets loose under large $\alpha$, which is caused by the extreme (small) propensity weights.}
    \vspace{-0.3cm}
    \label{fig:bound}
\end{figure*}

\begin{figure*}[hbt]
    \centering
    \includegraphics[width=0.85\linewidth]{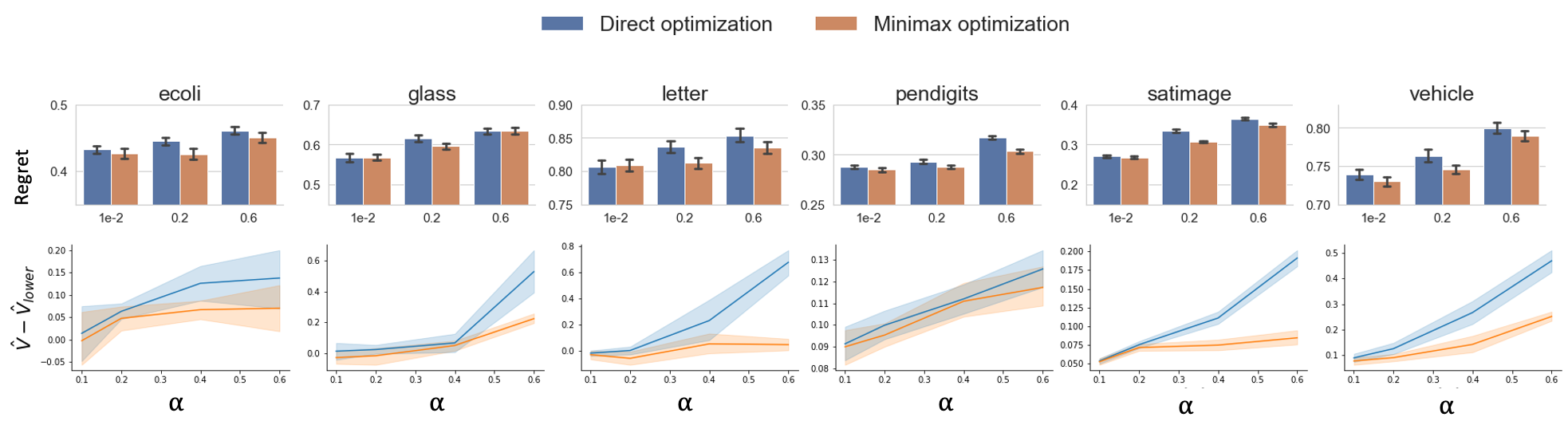}
    \caption{\footnotesize \textbf{Upper:} the regret (which equals $1-\text{reward}$ in our experiments) on the uncertainty-perturbed testing data under different $\alpha$; \textbf{Lower:} the robustness of the trained policies from standard off-policy learning ({\color{blue}{direct optimization}}) and our approach ({\color{orange}{minimax optimization}}), on the uncertainty-perturbed testing data, measured by how much they would fluctuate, i.e. $\hat{V}^*_{\text{DR}} - \underline{\hat{V}^*_{\text{DR}}}$, where $V^*$ denotes the trained policies.}
    \vspace{-0.5cm}
    \label{fig:fluctuation}
\end{figure*}

% \begin{figure*}[bht]
%     \centering
%     \includegraphics[width=\linewidth]{images/regret_bar.png}
%     \caption{Caption}
%     \label{fig:regret}
% \end{figure*}

% \qye{We need an experiment to illustrate our methods are robust to runtime uncertainties. For now, I think the existing results only show that they perform comparably to or outperform the competitors.}
We first conduct simulation experiments to examine the following questions:

\textbf{Q1:} how does the lower-bounding methods for RM, IPS and DR described in Section \ref{sec:bound} perform under different $\alpha$? 

\textbf{Q2:} does the proposed Algorithm \ref{algo:algo} improve the robustness of the learned policy against runtime uncertainty?

We also show the real-world performance of our robust off-policy learning approach by deploying the trained policy to an online e-commerce platform, where runtime uncertainties are frequent, for personalized product recommendation.
Due to the space limitation, we only show the key outcome in this paper and leave the complete numerical results in the appendix. All the results are obtained from ten repetitions.

% \begin{itemize}
%     \item the lower-bounding methods provide eligible surrogate functions for the estimators;
%     \item the minorize-maximization learning procedure improves the prediction robustness of the learnt policy, in terms of potential fluctuation, policy improvement and stability, under future uncertainty.
% \end{itemize}

% \subsection{Simulation}
% \label{sec:simulation}
{\textbf{Simulation}}.
We adopt the classical setting that generates bandit feedback according to multiclass classification \citep{dudik2011doubly,vlassis2019design}. In particular, a $k$-class classification task is turned into the $k$-arm contextual bandit problem. In the classification problem, the data $\{(x,c)\}$ for the classification task are i.i.d observations where $x \in \mathcal{X}$ is the context (feature) vector and $c \in \{1,\ldots,k\}$ is the class label. Here, each data point $(x,c)$ is converted into a cost-sensitive classification sample $(x, r_1,\ldots,r_k)$, where $r_a = I[a = c]$ is the 0-1 reward for predicting with label $a$. 

We now describe the data-generating mechanism. The feedback data under a given policy $\pi$ is constructed as follows. For each classification instance, we sample the label $a$ with the probability $\pi(a|x)$, and reveal the corresponding reward $r_a$. We use the same benchmark datasets from the UCI repository as in \citep{dudik2011doubly,vlassis2019design}, with the descriptions provided below. We design the logging policy as: $\pi_{0}(a|x) \propto \theta^{\intercal}_a x$ for all $a = 1,\ldots,k$, where $\theta_a$ are sampled i.i.d from the standard multivariate normal distribution. We use $\pi_0$ for off-policy learning.

\begin{center}
\resizebox{0.8\columnwidth}{!}{
    \begin{tabular}{|c|c|c|c|c|c|c|}
    \hline
        Dataset & Ecoli & Glass & Letter & PenDigits & SatImage & Vehicle \\ \hline \hline
        \#samples & 336 & 214 & 20000 & 10992 & 6435 & 846 \\ \hline
        \#classes & 8 & 6 & 26 & 10 & 6 & 4 \\ \hline 
    \end{tabular}%
    }
\end{center}

% \begin{table}[hbt]
% \footnotesize
%     \centering
%     \resizebox{\columnwidth}{!}{
%     \begin{tabular}{|c|c|c|c|c|c|c|}
%     \hline
%         Dataset & Ecoli & Glass & Letter & PenDigits & SatImage & Vehicle \\ \hline \hline
%         \#samples & 336 & 214 & 20000 & 10992 & 6435 & 846 \\ \hline
%         \#classes & 8 & 6 & 26 & 10 & 6 & 4 \\ \hline 
%     \end{tabular}%
%     }
%     \caption{\footnotesize Descriptive statistics for the datasets.}
%     \vspace{-0.3cm}
%     \label{tab:data}
% \end{table}

% Here, we reduce to the setting where the collected data is generated according to $\pi_d$, so we have $\pi_d=\pi_0$ in this case. However, during testing, the runtime uncertainty will perturb the execution of the candidate policy

\textbf{Adding Runtime uncertainty.} Since the real-world runtime uncertainty can depend on $a$ and $x$, given a policy $\pi$, we add noise to $\pi$ and obtain the uncertainty-injected policy from which the feedback data is actually generated:
\[
\tilde{\pi}(a|x):=\frac{\pi_{}(a|x) \cdot U_{a,x}(\alpha)}{\sum_{\tilde{a}}\pi_{}(\tilde{a}|x) \cdot U_{\tilde{a},x}(\alpha)},
\]
where $U_{a,x}(\alpha)$ is sampled from the truncated normal distribution with unit variance and mean $\gamma^{\intercal}_a x$, where $\gamma_a$ is also sampled from standard multivariate normal distributions. We set the truncation interval to be $\big[0, \exp(\alpha)\big]$. Then it is easy to check that $e^{-\alpha} \leq \tilde{\pi}_{}(a | x)\big/\pi_{}(a | x) \leq e^{\alpha}$ almost surely for all $a$ and $x$, which conforms to (\ref{eqn:perturb-def}).
% \[
% \pi_U(a|x) = \frac{\pi_{\text{design}}(a|x) \cdot U_{a,x}(\alpha)}{\sum_{\tilde{a}}\pi_{\text{design}}(\tilde{a}|x) \cdot U_{\tilde{a},x}(\alpha)},
% \]
% where the random noise $U_{a,x}(\alpha)$ are sampled i.i.d from the uniform distribution on $\big[0,\exp(\alpha)\big]$. In other words, after the deployment, the probability of choosing an action can only be changed up to a multiplicative constant. We design this mechanism for two reasons: \textbf{1}. the formulation satisfies our initial definition of the uncertainty in (\ref{eqn:perturb-def}); \textbf{2}. it is easy to verify that when $\pi_{\text{design}} = \pi_0$, which is here given by $1/k$, the perturbed policy is equivalent to the design policy in expectation:
% \[\Ebb_{U}[\pi_U(a|x)] = \pi_{\text{design}}(a|x) = 1/k.\]
% Hence, the perturbation mechanism that we consider here is practical and has realistic meaning.

% \textbf{Estimators.} We experiment with the IPS, RM and DR estimators as described in Section \ref{sec:prelim}. The model family of the RM estimator (and the RM part of DR), as well as the bounding functions $\hat{f}_a$ and $\hat{g}_a$, are given by the standard \emph{Regression Tree}. The tuning and other implementation details are left in the appendix.

\textbf{Estimators \& Experiment setting}. We experiment with the IPS, RM and DR estimators as described in Section \ref{sec:prelim}. The model family of the RM estimator (and the RM part of DR), as well as the bounding functions $\hat{f}_a$ and $\hat{g}_a$, are given by the standard \emph{Regression Tree}. The tuning and other implementation details are left in the appendix. We obtain the feedback data using the noise-injected $\tilde{\pi}_0$, do the train-validation-test split detailed in the appendix, and conduct off-policy estimation \& learning with $\pi_0$. To answer \textbf{Q1}, we first compute the RM estimator together with its bounding functions $\hat{f}_a$ and $\hat{g}_a$ using the training \& valuation data, and plot their values on the testing data (Figure \ref{fig:bound}). IPS does not require further training, so we directly report their values, as well as the corresponding (lower) bounds, on the testing data. We combine IPS and RM to obtain the results for DR. To answer \textbf{Q2}, we first conduct off-policy optimization using both the standard off-policy learning and the proposed minimax learning method, with DR as the estimator, to obtain the trained policies $\pi^*$. We then report their associated regrets on the testing data after adding the runtime uncertainty to $\pi^*$'s execution. We also study their robustness by measuring how much their values may fluctuate on the testing data, quantified by $\hat{V}_{\text{DR}}(\pi^*,\pi_0) - \underline{\hat{V}_{\text{DR}}}(\pi^*,\pi_0)$, computed on testing data.

\textbf{Simulation results and analysis.} We consider a wide range of uncertainty, i.e. $\alpha \in \{0.01, 0.2, 0.4, 0.6\}$. We first examine the proposed bounding methods for a given policy. Here, we use $\pi^*$ obtained from the standard off-policy learning. From the results in Figure \ref{fig:bound}, we first observe that the solutions to the ERM subproblem ($\hat{f}_a$ and $\hat{g}_a$) provide reasonable bounds for the RM method. The solutions to the proposed optimization problems in Section \ref{sec:bound} also reasonably provide lower bounds to the IPS and DR estimators. Other than the fact that IPS estimator suffers from variance issues, the performances of the bounding methods are generally consistent. It is also expected that a larger $\alpha$ leads to looser bounds for all the estimators. 

Next, we compare the testing performance and the robustness of $\pi^*$ optimized by standard off-policy learning and our approach. In particular, the testing regret is computed by first adding the runtime uncertainty to $\pi^*$.
Then we examine the robustness of $\pi^*$ by checking how much it might fluctuate via the gap of: $\hat{V}_{\text{DR}}(\pi^*,\pi_0) - \underline{\hat{V}_{\text{DR}}}(\pi^*,\pi_0)$. The gap provides a reasonable measurement since $\underline{\hat{V}_{\text{DR}}}(\pi^*,\pi_0)$ gives the worst-possible performance downgrade caused by runtime uncertainty.
The results are in Figure \ref{fig:fluctuation}. In the upper panel, we see that the proposed approach achieves better regrets for most cases, especially under larger $\alpha$. The simulation results justify the effectiveness of our learning approach, particularly as runtime uncertainty increases. 
% We consider $\hat{V}_{\text{DR}}(\hat{\pi},\pi_0) - \underline{\hat{V}_{\text{DR}}(\hat{\pi},\pi_0)}$ because if $\hat{\pi}$ is robust to runtime uncertainty, it should also fluctuate less under the offline perturbation which is designed to mimic the runtime uncertainty. 
In the lower panel of Figure {\ref{fig:fluctuation}}, we see the proposed algorithm significantly improves the robustness of the optimized policy $\pi^*$, and the degree of improvement increases with $\alpha$. In the appendix, we provide in-depth analysis of the impact of $\alpha$ on off-policy learning, as well as the comparisons with the \emph{oracle} regret. 
% In particular, we consider $\hat{V}_{\text{DR}}(\hat{\pi},\pi_0) - \underline{\hat{V}_{\text{DR}}(\text{\pi}),\pi_0}$ on testing data as the measure of robustness for $\hat{\pi}$. This makes sense because if the $\hat{\pi}$ is sensitive to runtime uncertainty,  From the results in Figure {\ref{fig:fluctuation}}, we observe that the proposed minorize-maximization procedure significantly improves the robustness of the optimized policy against the baseline, and the degree of improvement increases with $\alpha$. Finally, we show the regret of the optimized policy on the testing data in Figure \ref{fig:fluctuation}. It is evident that the proposed learning approach leads to prediction-robust policy while still making policy improvements for most of the cases. The results suggest the superiority of the proposed approaches compared with the original off-policy learning in terms of the prediction robustness when the testing environment is subject to uncertainty. 
% \subsection{Real-world experiments}
% \label{sec:real-experiment}
% \textbf{Real-world experiments}
\indent \textbf{Real-world deployment and analysis}.
We conduct both offline and online experiments via the production platform of a major e-commerce website in the U.S. We use contextual bandit to personalize the daily homepage product recommendations, and we consider the click-through rate (CTR) as the reward. We leave the background detail, offline results and analysis in the appendix, and discuss the online testing below. 
% In our offline experiments, we use the truncated and normalized propensity scores which are more suitable for real-world logging policies. 
We conduct online A/B testing to compare the proposed approach and the standard off-policy learning; both rely on the DR method with properly truncated propensity scores. The value of $\alpha$ is treated as a tuning parameter and selected via the validation data. For a typical e-commerce platform, runtime uncertainties can be induced by \emph{special offers}, \emph{stock availability}, \emph{upselling events}, and \emph{infrastructural malfunction} with the pipeline, caching, front-end computation and interleaving experiments. Trained under the same data and model family, during the 21-day testing period, the proposed approach consistently outperforms the standard off-policy learning, which often suffers from the previously identified runtime uncertainties. It indicates that the robustness of our solution leads to improved online performances.

% We provide the details of the online experiments in the appendix. We see in Figure \ref{fig:real-data} that the proposed approach outperforms the standard off-policy learning in both the uncertainty-aware offline evaluation (similar to what we did in simulation) and the online experiments. The real-world experiments further demonstrate the effectiveness of our approach against runtime uncertainty.

\begin{figure}[htb]
    \centering
    \includegraphics[width=0.85\linewidth]{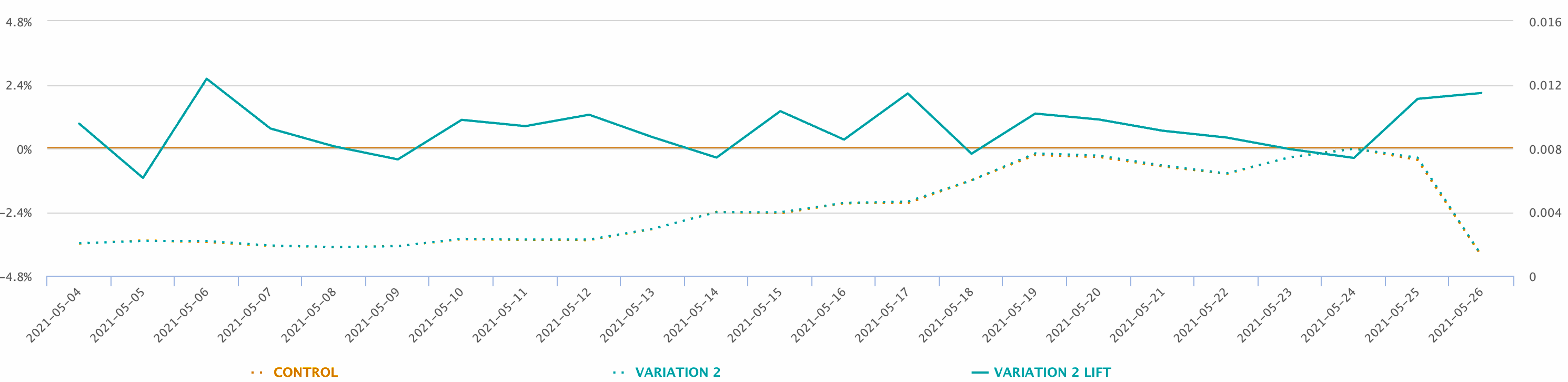}
    \caption{\footnotesize  Real-world online experiment results. {\color{orange}{CONTROL}} refers to the standard off-policy learning, and {\color{green}{VARIATION}} is the proposed robust off-policy learning. The right and left y-axis gives the daily CTR, and the percentage of lift against the control model. }
    \vspace{-0.7cm}
    \label{fig:ab-test}
\end{figure}

\section{Conclusion}\label{sec:conclusion}

We study the novel problem of robust off-policy learning for runtime uncertainty. We propose a principled solution with max-min learning and justify the theoretical implications and guarantees. Our solution is examined via simulation and real-world testings. In the presence of runtime uncertainty, our approach compares favorably to standard off-policy learning. We hope our work promotes future research on practical post-deployment robustness of AI solutions. For this particular topic, our framework can easily extend to diverse problem settings with suitable choices of uncertainty sets.

% We address the novel problem of robust off-policy learning under testing-time uncertainty. The proposed methods are built on solid theoretical arguments and have the desired computation efficiency. \\
% \textbf{Scope and limitation}. We adopt the frequentists' minimax-risk view and define the uncertainty in a deterministic way. While they lead to solid theoretical arguments and computation-efficient solutions, the outcome can be conservative. Controlling the average-case uncertainty is an alternative approach, where the Bayesian risk setting can be explored if a reasonable prior is available, which we leave to the future work. \\
% \textbf{Tightness of the lower bound}. The quality of the surrogate function depends on the tightness of the lower bound. For RM and DR in particular, they rely on the auxiliary ERM and will thus experience similar bias-variance tradeoff as the supervised learning. \\
% \textbf{Convergence and optimality}. When the objective function is convex, minorize-maximization algorithm converges to local stationary points \cite{lange2016mm}. However, our objective can be non-convex in general, so there is no guarantee that we find the local optimum. Nevertheless, we set the stopping criteria of Algorithm \ref{algo:algo} to ensure a non-decreasing objective value during the updates. 
% In practice, we find small learning rates favorable.

% \newpage
% \subsubsection*{References}
% \bibliographystyle{abbrv}

\bibliography{references}

\clearpage
\appendix
\onecolumn
\section*{Appendix}
In \ref{sec:variant-ips}, we discuss the robust off-policy learning with the truncated and normalized IPS. We then provide the proof for Lemma 1, Lemma 2 and Theorem 1 in A.2 and A.3. We provide comprehensive interpretation for our max-min optimization in A.4. In A.5, we present the detailed experiment setup and additional numerical results. The implementation of our simulation studies is provided as separate supplement material for the submission.

\section{Methods for the truncated and normalized IPS}
\label{sec:variant-ips}

The most common variants of the IPS estimator are the truncated IPS and normalized IPS \cite{swaminathan2015self}. Even though they are biased estimators, in many cases, they significantly reduces the variances and thus provide a improved bias-variance tradeoff. Since they are not the focus of our paper, we refer the readers to \cite{vlassis2019design} for reference on understanding the design of off-policy estimators. 

It is obvious that the proposed lower-bounding approach and minorize-maximization procedure are fully compatible with the \emph{truncated IPS} (and when it is used to also reduce the variance DR), by simply truncating the constraint sets at the given threshold. For the lower-bounding problem, the objective and constraints are still separable on each sample $i$, and there is no increase in the computation complexity:

\begin{equation}
\label{eqn:optimization-TIPS}
\begin{split}
    & \qquad \qquad \qquad \text{minimize} \quad \frac{1}{n} \sum_{i=1}^n  \frac{\pi(a_i|x_i)}{p(a_i | x_i)} r_i \\
    & \text{s.t.} \quad e^{-\alpha} \max\big\{     q, \pi_{0}(a_i | x_i) \big\} \leq p(a_i | x_i) \leq \min\big\{ e^{\alpha} \pi_{0}(a_i | x_i), 1 \big \},
\end{split}
\end{equation}
where $q$ is the selected cutoff to reflect the minimum propensity score allowed. 
Compared with Equation (13) in the paper, which gives the constraint optimization problem for lower-bounding the IPS, the second set of constraints is removed here as we no longer require:
\[\sum_{a \in \mathcal{A}} p(a|x_i) \leq 1, \text{ for } i=1,\ldots,n.
\]
Our considerations are two folds: 
\begin{itemize}
\item the above normalization constraint will inevitably lead to small propensity scores $p(a|x_i)$, which opposes the inital goal of avoiding large variance;
\item truncated IPS is known to be biased, so we are less concerned about $p(a|x)$ not being properly normalized as long as we are able to reduce the variance. 
\end{itemize}
We do not experiment with the truncated IPS in simulation, because the benchmark datasets are relatively small so we find out that the results can be overly sensitive to the cutoff, and are thus not representative of the general case when applying the truncated IPS. 

The normalized propensity score method (nIPS) is another common control covariate estimator that often gives an improved bias-variance tradeoff for the IPS estimator:
\begin{equation}
    \hat{V}_{\text{nIPS}}(\pi;\pi_0) = \frac{1}{C(\pi, \pi_0)} \sum_{i=1}^n \frac{\pi(a_i|x_i)}{\pi_0(a_i|x_i)} r_i,
\end{equation}
where $C(\pi, \pi_0) = \sum_{i=1}^n \frac{\pi(a_i|x_i)}{\pi_0(a_i|x_i)}$ is the control covariate. The nIPS can be helpful when $\pi$ is very different from $\pi_0$ or when $\pi_0(a|x)$ gets very small, because IPS can have large variance under such circumstances. Recall that our goal is to lower-bound $\hat{V}_{\text{nIPS}}(\pi;\pi_U)$, where $\pi_U$ is subject to the same constraints as in (\ref{eqn:optimization-TIPS}) or Equation (13) in the paper. 

Again, we consider using the change of variable: $w = 1 / p$, where $p(a_i | x_i)$ is the optimization policy that takes the place of $\pi_U$ in our optimization problem. The lower bound of $\hat{V}_{\text{nIPS}}(\pi)$ is then given by:
\begin{equation}
\label{eqn:optimization-nIPS}
\begin{split}
    & \qquad \qquad \qquad \text{minimize} \quad \frac{\sum_{i=1}^n w(a_i, x_i)\pi(a_i|x_i)}{\sum_{i=1}^n w(a_i, x_i)\pi_0(a_i|x_i)}r_i \\ 
    & \text{s.t.} \quad \frac{e^{-\alpha}}{\pi_0(a_i|x_i)} \leq w(a_i,x_i) \leq \frac{e^{\alpha}}{\pi_0(a_i|x_i)}, \quad  \sum_{a\in \mathcal{A}} \frac{1}{w(a_i,x_i)} \leq 1, \text{ for } i=1,\ldots,n.
\end{split}
\end{equation}

The objective function in (\ref{eqn:optimization-nIPS}) is quasiconvex, and the constraints are convex. Therefore, obtaining the minimum value is feasible in theory. However, the number of optimization parameters is $\mathcal{O}(nk)$, and unlike the cases for IPS and DR where the objective and constraints are separable with respect to each instance $i$, the objective function here is not separable. As a consequence, the exact computation for lower-bounding nIPS is impractical, especially on large datasets, due to the dependency on $n$. We also point out that the second set of constraints enforces the optimization policy to be a valid PDF. Unlike the truncated IPS, we keep it here for completeness.

To solve the computation issue, we come up with a possible two-step \emph{greedy} workaround method for lower-bounding the nIPS: 
\begin{itemize}
    \item first compute the lower bound for the IPS to obtain the optimal \emph{optimization policy} $p^*$ and the corresponding lower bound $\underline{\hat{V}_{\text{IPS}}}(\pi)$;
    \item then let $n\underline{\hat{V}_{\text{IPS}}}(\pi)/C(\pi,p^* )$ be the lower bound for $\hat{V}_{\text{nIPS}}(\pi)$.
\end{itemize}

The advantage of the proposed workaround is that the computation complexity is now almost the same as the IPS method. Importantly, we are able to parallelize the computation and get rid of the dependency on $n$. However, the gap between the workaround solution and the original solution is not controlled or bounded, which is clearly a disadvantage of the greedy approach. 

For illustration, we experiment with the workaround method for nIPS in the off-policy optimization problem. Similar to the setting for Figure 4(lower) in the paper, we examine the robustness of the learnt policy via the prediction fluctuation $\hat{V}(\pi_{\theta}) - \underline{\hat{V}}(\pi_{\theta})$ on testing data. The results on several datasets are provided in Table \ref{tab:appen-nIPS}, where we also list the results of DR for direct comparisons. We see that when compared with the DR, the robustness for the workaround method of nIPS is somewhat weaker, which is likely to be caused by the following two reasons:
\begin{itemize}
    \item the inaccurate approximations caused by the greedy workaround method when computing the lower bound of nIPS;
    \item the simulation setting does not bring out the advantages of the variance-reduction methods since in expectation, our design will not induce extreme propensity scores.
\end{itemize}

\begin{table*}[htb]
    \centering
    \resizebox{\columnwidth}{!}{
    \begin{tabular}{c|ccc|ccc|ccc}
    \hline
        Dataset & \multicolumn{3}{c}{Glass} & \multicolumn{3}{c}{Ecoli} & \multicolumn{3}{c}{SatImage}  \\ \hline
        $\alpha$ & 0.2 & 0.4 & 0.6 & 0.2 & 0.4 & 0.6 & 0.2 & 0.4 & 0.6 \\  \hline
        {\textbf{nIPS}} & 0.0333 & 0.1081 & 0.1896  & 0.0634  & 0.0883  & 0.1077 &  0.0762 & 0.1049 & 0.1215 \\ 
        & (.0197) & (.0157) & (.0232) & (.0177) & (.0231) & (.0286) & (.0153) & (.0280) & (.0333) \\ \hline
        {\textbf{DR}} & 0.0173 & 0.0667  &  0.157  & 0.0427  & 0.0635  & 0.0757 & 0.0681 & 0.0754 & 0.0823 \\ 
        & (.0109) & (.0133) & (.0162) & (.0129) & (.0153) & (.0168) & (.0126) & (.0141) & (.0164) \\ \hline
    \end{tabular}
    }
    \caption{The prediction fluctuation $\hat{V} - \underline{\hat{V}}$ on testing data, for the nIPS and DR method.}
    \label{tab:appen-nIPS}
\end{table*}

In the real-world experiments where extreme propensity scores do occur (Figure 5 in the paper), we observe much improved performances from the truncated and normalized IPS methods (Appendix \ref{sec:real-world}). The original IPS method, on the other hand, experiences variance so large that putting it in the same plot with the other methods is impractical.

\section{The derivation of lower bound of APV}\label{sec:lb_apv}

We aim at deriving bounds that are agnostic to the specific form of $U$, for which we consider the most adversarial case where the added noise $U(a,x)$ depends on both $a$ and $x$.

The key insight is to reveal how the potential value $R(a)$ interacts implicitly with the perturbation $U(a,x)$, if the data were generated under $\pi_U$. The condition $e^{-\alpha}\leq \pi_0 / \pi_U \leq e^{\alpha}$ provides a special characterization of $U(a,x)$, since by applying the Bayes rule, it holds for any two actions $a$ and $a'$ that:
\begin{equation}
\label{eqn:bound-U}
    \exp(-2\alpha) \leq \frac{p(U(a,x) = u | A = a', X = x)}{p(U(a,x) = u | A = a, X = x)} \leq \exp(2\alpha).
\end{equation}
We defer the derivations to the appendix. Notice that it is a critical result to have $U(A,X)$ implicit bounded as above, since $R(A,X) \perp A \, \big|\, X, U(A,X)$ by assuming the feedback data is generated under $\pi_U$. Therefore, for any given $x$, we may conjecture that $R(a,x)$ do not deviate much from $R(a',x)$ for $a\neq a'$.

% Therefore, it is reasonable to conjecture that the reward model can also be bounded when the uncertainty is controlled using (\ref{eqn:perturb-def}). 
Specifically, we need to characterize $\Ebb\big[R(a',x) \,\big|\, A = a, X = x]$ as suggested by the formulation of APV, since we only get to observe $R(a,x)$ when the chosen action is $a$. 

The trick we employ here is to express the unobservable quantities as a function of the estimable objects, using the Radon–Nikodym derivative as a change-of-measure tool. Let $\diff P_{R(a,x) |A, X}$ be the density for random variable $R(a,x)\big|A, X$ with respect to the underlying measure that we presume to rely on $\pi_U$. Using the Radon–Nikodym derivative \citep{durrett2019probability}, for all $a, a' \in \mathcal{A}$, it holds that:
\begin{equation}
\label{eqn:R-N}
\begin{split}
    & \Ebb[R(a,x) \big| A = a', X = x] = \Ebb \Big[R(a,x)\frac{\diff P_{R(a,x)|A=a', X=x}(R(a,x))}{\diff P_{R(a,x)|A=a, X=x}(R(a,x))} \,\Big|\, A = a, X = x \Big].
\end{split}
\end{equation}
Since $\Ebb[R(a,x) \big| A = a, X = x]$ is estimable, we just need to bound the derivative ratio term to obtain the overall bound.
We denote the ratio by the shorthand: $D(r, x, a, a') := \diff P_{R(a,x)|A=a', X=x}(r) / \diff P_{R(a,x)|A=a, X=x}(r)$. 

Note that when no perturbation is added, i.e. $\alpha = 0$, we have $D(r_1, x, a, a') = D(r_2, x, a, a')$ for any two $r_1,r_2$, so the problem reduces to the standard off-policy setting. In the following lemma, we first prove the existence of $D$ under $\pi_U$, and show that under the constraints $e^{-\alpha} \leq \pi_{0}/\pi_u \leq e^{\alpha}$, the ratio $D(r_1, x, a, a')$ is still close to $D(r_2, x, a, a')$ depending on $\alpha$. 

\begin{lemma}
\label{lemma:RN_bound}
When the underlying measure $P_{R(a,x)|A, X}$ is subject to $\pi_U$ that satisfy (\ref{eqn:perturb-def}), the Radon–Nikodym derivative $D(r, x, a, a')$ exists for all $a, a' \in \mathcal{A}$ and $x \in \mathcal{X}$.
Furthermore, suppose that $\pi_U$ is a continuous function of any $U=u$, then the following inequality holds for almost every $a, a'$ and $x$:
\begin{equation}
\label{eqn:RN-derivative-bound}
\frac{D(r_1, x, a, a')}{D(r_2, x, a, a')} \leq \exp(2\alpha), \, \forall r_1, r_2.
\end{equation}
\end{lemma}

The proof is relegated to Appendix A.2.1 as below. The above lemma converts the uncertainty induced by $U(a,x)$ to a bound on the ratio $D$, which is a significant step towards our goal because it reveals an explicit constraint for bounding the potential value according to (\ref{eqn:R-N}). Also, it is straightforward to verify that the derivative ratio $D$ carries two constraints by itself:
\begin{equation}
\label{eqn:RN-constraint}
\begin{split}
    & D(r, x, a, a') \geq 0, \\
    & \Ebb \big[D(R(a,x), x, a, a') \,\big|\, A = a, X = x] = 1.
\end{split}
\end{equation}
Therefore, lower-bounding RM is now converted to a constraint optimization problem where we minimize (\ref{eqn:R-N}) under (\ref{eqn:RN-derivative-bound}) and (\ref{eqn:RN-constraint}), where $D$ is now treated as optimization variable: $\min_{D \text{ s.t. }(\ref{eqn:RN-derivative-bound}),(\ref{eqn:RN-constraint})}\Ebb \Big[R(a,x) D\big(R(a,x), x, a, a'\big) | A=a, X=x \Big].$

Applying the standard strong duality arguments, we find that the minimizer of the above problem can be obtained as the solution of an \emph{auxiliary} ERM, which is essentially a subproblem of ERM, which we describe below.

\begin{lemma}
\label{lemma:apv-eqv+}
The optimizer of: 
\[
\min_{D: \,D \text{ s.t. }(\ref{eqn:RN-derivative-bound}),(\ref{eqn:RN-constraint})}\Ebb \Big[R(a,x) D\big(R(a,x), x, a, a'\big) | A=a, X=x \Big],
\]
is the solution to the following problem:
\begin{equation}
\label{eqn:ERM}
    \min_{f_{a, a'}(\cdot)} \Ebb \Big[\ell _{\alpha} \Big(R(a,x) \, , \, f_{a, a'}(x)\Big) \,\Big|\, A=a, X=x \Big],
\end{equation}
where the loss function $\ell_{\alpha}$ is specified by:
\begin{equation}
\label{eqn:auxiliary-loss}
\begin{split}
    & \ell _{\alpha} \Big(R(a,x) \, , \, f_{a, a'}(x))\Big) = \Big\{ R(a,x)  - f_{a, a'}(x) \Big\}^2_{+} + e^{2\alpha}  \Big\{ R(a,x) - f_{a, a'}(x) \Big\}^2_{-}.
\end{split}
\end{equation}
Here, $\{\cdot\}_{+}$ and $\{\cdot\}_{-}$ gives the positive and negative part. Since the potential value $R(a,x)$ are observed given $A=a$, the above setting describes a subproblem of the ERM.
\end{lemma}

\begin{figure}[bht]
    \centering
    \includegraphics[width=0.7\linewidth]{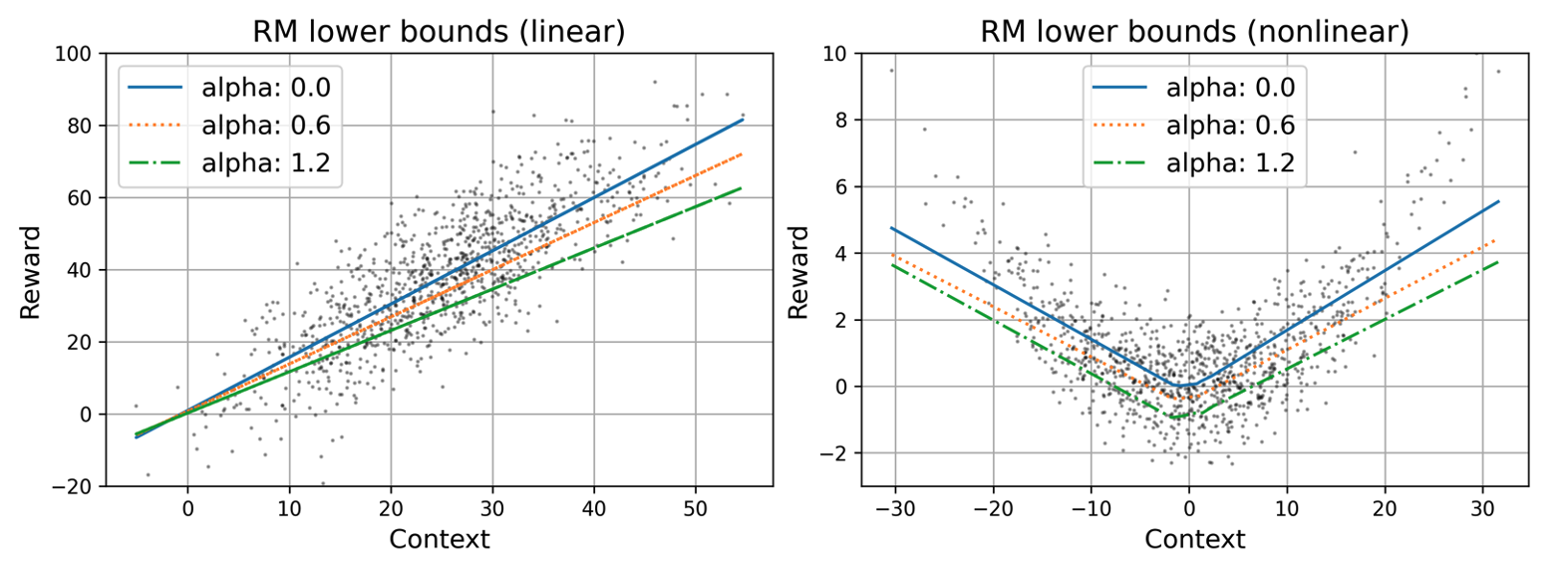}
    \caption{The solutions of the auxiliary (subproblem of) ERM under different $\alpha$, for both the linear and non-linear data-generating mechanisms.}
    \label{fig:lower-bounds}
\end{figure}

The proof is given below in Appendix A.2.2. When solving the auxiliary ERM, we may choose $f_{a, a'}(\cdot)$ from a parametric model family (e.g. gradient boosting tree) that fits the context space. As a sanity check, when $\alpha=0$, i.e. no uncertainty is asserted, $\ell_{\alpha}$ is simply the least square loss so we reduce to the standard ERM setting. When $\alpha > 0$, the loss function puts more weight on the negative part and the optimized response surfaces are therefore "shifted" downward and thus produce the lower bounds. See Figure \ref{fig:lower-bounds} for visual illustrations of both the linear and non-linear cases.

\subsection{Proof of Lemma \ref{lemma:RN_bound}}

\begin{proof}
For notation simplicity, we omit the explicit dependency of $U$ on $a$ and $x$.

First, we show the existence of the ratio 
\[ 
D(r, x, a, \tilde{a}) = \frac{\emph{d} P_{R(a,x)|A=\tilde{a}, X=x}(r)}{\emph{d} P_{R(a,x)|A=a, X=x}(r)}
\]
for all $a, \tilde{a} \in \mathcal{A}$, $x \in \mathcal{X}$ and $r \in \mathbb{R}$, when the measure $P_{R(a)|A=a, X=x}(r)$ is induced by $\pi_U$ with $\exp(-\alpha) \leq \pi_0(a|x) / \pi_U(a|x) \leq \exp(\alpha)$. 
Assume that there exists at least one setting $u^* \in \mathcal{U}$ for every $x\in \mathcal{X}$ such that $\pi_{u^*}(a|x) = \pi(a|x)$.

For any $u \in \mathcal{U}$ such that $p(U=u | A=a) > 0$, by applying the Bayes rule, we have:
\begin{eqnarray}
\label{eqn:appen-1}
  p(U =u | A=a, X=x) = \frac{p(A=a | U =u, X=x) p(U = u | X=x)}{p(A=a | X=x)}
\end{eqnarray}
and 
\begin{eqnarray}
\label{eqn:appen-2}
  p(U = u | A=\tilde{a}, X=x) = \frac{p(A=\tilde{a} | U = u, X=x) p(U = u | X=x)}{p(A=\tilde{a} | X=x)}.
\end{eqnarray}
Combining (\ref{eqn:appen-1}) and (\ref{eqn:appen-2}) we get:
\begin{eqnarray*}
  \frac{p(U = u | A = \tilde{a}, X = x)}{p(U = u | A = a, X = x)} = \frac{p(A = \tilde{a} | U = u, X = x) p (A = a | X=x)}{p (A = \tilde{a} | X=x)p(A = a | U = u, X = x)}.
\end{eqnarray*}
By the definition of $\alpha$-degree uncertainty in Equation (2), we have $p (A = a | X=x) / p(A = a | U = u, X = x) = p (A = a | U = u^*, X=x) / p(A = a | U = u, X = x) \in [\exp(-\alpha), \exp(\alpha)]$ (where $u^*$ is the setting that recovers the original policy). Therefore, we conclude that for almost all $u \in \mathcal{U}$:
\begin{equation}
\label{eqn:appen-3}
    \exp(-2\alpha) \leq \frac{p(U = u | A = \tilde{a}, X = x)}{p(U = u | A = a, X = x)} \leq \exp(2\alpha).
\end{equation}
Hence, $p(U | A=\tilde{a}, X=x)$ is mutually absolute continuous with respect to $p(U | A=a, X=x)$, which allows us to further define the likelihood ratio:
\begin{equation}
    L(u, x, a, \tilde{a}) \equiv \frac{p(U=u |A=\tilde{a}, X=x )}{p(U=u |A=a, X=x )}.
\end{equation}
Similarly, by applying the Bayes rule, it can be shown that for almost all $u_1,u_2 \in \mathcal{U}$, we have:
\begin{equation}
\label{eqn:appen-eqn-five}
    \exp(-2\alpha) \leq \frac{L(u_1, x, a, \tilde{a})}{L(u_2, x, a, \tilde{a})} \leq \exp(2\alpha).
\end{equation}
Recall that the \emph{potential value} $R(A,X)$ is conditionally independent of action $A$ given context $X$ and $U$. Hence, for any set $S \in \mathbb{R}$:
\[ 
\Ebb_{\pi_U} \big[ \Ibf (R(a,x) \in S) \,|\, U, A=\tilde{a}, X=x \big] = \Ebb_{\pi_U} \big[\Ibf(R(a,x) \in S) \,|\, U, X=x \big].
\]
Then, by applying the tower property of conditional expectation, we obtain:
\begin{equation}
\label{eqn:appen-expectation-expand}
\begin{split}
 & \Ebb_{\pi_U} \big[\Ibf(R(a,x) \in S) \,|\, A=\tilde{a}, X=x \big]\\ 
 & = \Ebb_{\pi_U} \Big[ \Ebb_U \big[ \Ibf(R(a,x) \in S) \,|\, U, A=\tilde{a}, X=x \big]  \,\big|\, A=\tilde{a}, X=x \Big] \\ 
 & = \Ebb_{\pi_U} \Big[ \Ebb_U \big[ \Ibf(R(a,x) \in S) \,|\, U, X=x \big]  \,\big|\, A=\tilde{a}, X=x \Big]\\ 
 & = \Ebb_{\pi_U} \, \Ebb_{U} \Big[ \frac{\text{d}P_{U|A=\tilde{a},X=x}}{\text{d}P_{U|A=a,X=x}} \Ebb_U \big[ \Ibf(R(a,x) \in S) \,|\, U, X=x \big] \,\big|\, A=a, X=x \Big] \\ 
 & = \Ebb_{\pi_U} \, \Ebb_{U} \Big[L(U, x, a, \tilde{a}) \Ebb_U \big[ \Ibf(R(a,x) \in S) \,|\, U, X=x \big]   \,\big|\, A=a, X=x \Big] \\
 & = \Ebb_{\pi_U} \, \Ebb_{U} \big[L(U, x, a, \tilde{a})  \Ibf(R(a,x) \in S) \,\big|\, A=a, X=x \big].\\ 
\end{split}
\end{equation}
Notice that $L(u, x, a, \tilde{a}) \in [e^{-2\alpha}, e^{2\alpha}]$ almost everywhere by (\ref{eqn:appen-3}), so we have: 
\begin{equation}
    \exp (-2\alpha) \leq \frac{\Ebb_{\pi_U} \big[\Ibf(R(a,x) \in S) \,|\, A=\tilde{a}, X=x \big]}{\Ebb_{\pi_U} \big[\Ibf(R(a,x) \in S) \,|\, A=a, X=x \big]} \leq \exp(2\alpha).
\end{equation}
By the definition, $P_{R(a,x)|A=\tilde{a}, X=x}(S) = \Ebb_{\pi_U} \big[\Ibf(R(a,x) \in S) \,|\, A=\tilde{a}, X=x \big]$, and $P_{R(a,x)|A=a, X=x}(S) = \Ebb_{\pi_U} \big[\Ibf(R(a,x) \in S) \,|\, A=a, X=x \big]$, 
plus the previously shown result that $P_{R(a,x)|A=\tilde{a}, X=x}$ and $P_{R(a,x)|A=a, X=x}$ are mutually absolute continuous, we conclude the existence of the ratio $D(r, x, a, \tilde{a})$ for all $a, \tilde{a} \in \mathcal{A}$, $x \in \mathcal{X}$ and $r \in \mathbb{R}$.

In the next step, notice that the last line of (\ref{eqn:appen-expectation-expand}) can be further expanded into:
\begin{equation}
\label{eqn:appen-8}
    \Ebb_{\pi_U} \Big[ \Ebb_{U} \big[ L(U, x, a, \tilde{a}) \,|\, R(a,x), A=a,X=x  \big] \Ibf(R(a,x) \in S)   \,\big|\, A=a, X=x \Big].
\end{equation}
Then according to the Radon-Nikodym Theorem \cite{durrett2019probability}, together with the conclusion from (\ref{eqn:appen-8}) and (\ref{eqn:appen-expectation-expand}), we reach the following equality
\begin{equation*}
\begin{split}
    & \Ebb_{\pi_U} \big[\Ibf(R(a,x) \in S) \,|\, A=\tilde{a}, X=x \big] \\ 
    & = \Ebb_{\pi_U} \Big[ \Ebb_{U} \big[ L(U, x, a, \tilde{a}) \,|\, R(a,x), A=a,X=x  \big] \Ibf(R(a,x) \in S)   \,\big|\, A=a, X=x \Big].
\end{split}
\end{equation*}
Given the fact that 
\[ 
\Ebb_{\pi_U} \big[\Ibf(R(a,x) \in S) \,|\, A=\tilde{a}, X=x \big]  = \Ebb_{\pi_U} \big[D(R(a,x), x, a, \tilde{a})\Ibf(R(a,x) \in S) \,|\, A=a, X=x \big],
\]
we finally obtain: 
\begin{equation}
\label{eqn:appen-eqn-nine}
    D(R(a,x), x, a, \tilde{a}) = \Ebb_{U} \big[ L(U, x, a, \tilde{a}) \,|\, R(a,x), A=a,X=x  \big], \, \text{almost surely.}
\end{equation}
Under the assumption from the Lemma statement that  $\pi_u(a|x)$ is continuous with respect to all $u\in \mathcal{U}$, for any $\delta > 0$ we can always find a path $u(\delta)$ such that $L\big(u(\delta), x, a, \tilde{a}\big) < \inf_{u} L(u, x, a, \tilde{a}) + \delta$ and $\lim_{\delta \to 0} u(\delta) = u$, for each $u \in \mathcal{U}$. We then show the boundedness of $D(r, x, a, \tilde{a})$ in the following two steps. 

Firstly, according to (\ref{eqn:appen-eqn-nine}) and (\ref{eqn:appen-eqn-five}), for almost all $r_1 \in \mathbb{R}$, we have:
\begin{equation*}
\begin{split}
    D(r_1, x, a, \tilde{a}) & = \Ebb_{U} \big[ L(U, x, a, \tilde{a}) \,|\, R(a,x)=r_1, A=a,x=x  \big] \\ 
    & = L\big(u(\delta),x,a,\tilde{a}\big) \Ebb \Big[\frac{L(U,x,a,\tilde{a})}{L\big(u(\delta),x,a,\tilde{a}\big)}   \,\big|\, R(a,x)=r_1, A=a, x=x\Big] \\ 
    & \leq \exp(2\alpha) L\big(u(\delta),x,a,\tilde{a}\big).
\end{split}
\end{equation*}

Secondly, by the above construction, for all almost all $r_2 \in \mathbb{R}$ we have:
\[ 
D(r_2, x, a, \tilde{a}) \geq \inf_{u} L(u,x,a,\tilde{a}) > L(u(\delta), x, a, \tilde{a}) - \delta.
\]

Together, we obtain: 
\[D(r_1, x, a, \tilde{a}) \leq \exp(2\alpha)L(u(\delta), x, a, \tilde{a}) < \exp(2\alpha) \big( D(r_2, x, a, \tilde{a}) + \delta \big).\]
By taking $\delta \to 0$, we reach the conclusion that for almost all $r_1, r_2 \in \mathbb{R}$,
\[ 
\frac{D(r_1, x, a, \tilde{a})}{D(r_2, x, a, \tilde{a})} \leq \exp(2\alpha),
\]
which completes the proof.
\end{proof}

\subsection{Proof of Lemma \ref{lemma:apv-eqv+}}
The proof consists mostly of standard convex optimization arguments.
\begin{proof}
Together with the two regularity constraints on the Radon-Nikodym derivative ratio term $D$, as well as the results from Lemma 1, finding the upper bound (the lower bound can be obtained analogously) for the RM can be given by a constraint optimization problem:
\begin{equation}
\label{eqn:potential-outcome-upper}
\begin{split}
  & \sup_{D \in \mathcal{D}} \Ebb \big[R(a,x)D(R(a,x), x, a, \tilde{a}) \big| A = a, X = x \big], \\ 
    & \text{s.t.}\quad \Ebb \big[D(R(a,x), x, a, \tilde{a}) | A = a, X = x\big] = 1, \\
    &\quad \quad 0 \leq D(r_1, x, a, \tilde{a}) \leq e^{2\alpha} D(r_2, x, a, \tilde{a}), \, \forall r_1,r_2,
\end{split}
\end{equation}
where the set $\mathcal{D}$ gives all the possible Radon-Nikodym derivative ratios under the standard measurable conditions. 

Applying the standard functional duality argument, the above optimization problem becomes:
\begin{equation}
\label{eqn:lagrangian}
\begin{split}
  & \inf_{\lambda} \sup_{D \in \mathcal{D}} \Ebb \big[(R(a,x)-\lambda)D(R(a,x), x, a, \tilde{a}) \big| A = a, X = x \big] + \lambda, \\ 
    & \text{s.t.}\quad 0 \leq D(r_1, x, a, \tilde{a}) \leq e^{2\alpha} D(r_2, x, a, \tilde{a}), \, \forall r_1,r_2.
\end{split}
\end{equation}
Notice that the-above result is led by the generalized Slater's condition, because the constraint set is obviously convex, and there exists a feasible $D \in \mathcal{D}$ that satisfies the constraint, i.e. $D(R(a,x), x, a, \tilde{a})\equiv 1$ is in $\mathcal{D}$ with:
\begin{equation*}
\begin{split}
    & \Ebb\big[D(R(a,x), x, a, \tilde{a})| A = a, X = x\big] = 1, \\
    & D(R(a,x), x, a, \tilde{a}) \geq 0, \\
    & D(r_1, x, a, \tilde{a}) \leq e^{2\alpha} D(r_2, x, a, \tilde{a}).
\end{split}
\end{equation*}
Therefore, the strong duality result holds \cite{boyd2004convex}. It is straightforward to verify that the supremum of the inner problem in (\ref{eqn:lagrangian}) is attained by 
\[
D^*(R(a,x), x, a, \tilde{a}) = C\cdot 1[R(a,x) - \lambda \geq 0 ] + C\cdot e^{2\alpha} 1[R(a,x) - \lambda 
\leq 0 ],
\]
where $1[\cdot]$ is the indicator function and $C\geq 0$ can be any parameter induced by the dual problem. Plugging this $D^*$ into the original problem, we obtain:
\begin{equation}
\begin{split}
\inf_{\lambda} \sup_{C\geq 0} \Ebb \big[C\cdot1[R(a,x) - \lambda \geq 0] + C\cdot e^{2\alpha} 1[R(a,x) - \lambda < 0] \,\big|\, A = a, X = x \big] + \lambda.
\end{split}
\end{equation}
The above objective can be directly converted to:
\begin{equation}
\begin{split}
    & \qquad \qquad \qquad \qquad \qquad \qquad \inf \lambda \\
    & \text{s.t} \quad \Ebb \big[1[R(a,x) - \lambda \geq 0] + e^{2\alpha} 1[R(a,x) - \lambda \leq 0] \,\big|\, A = a, X = x \big] \geq 0.
\end{split}
\end{equation}
Since the function $q(\lambda) := 1[R(a,x) - \lambda \geq 0] + e^{2\alpha} 1[R(a,x) - \lambda < 0] $ is a non-decreasing monotone function of $\lambda$, the optimal solution $\lambda^*$ should be the only root of the $\Ebb \big[1[R(a,x) - \lambda \geq 0] + e^{2\alpha} 1[R(a,x) - \lambda \leq 0] \,\big|\, A = a, X = x \big]$ for each $a$ and $x$. Therefore, solving the-above optimization problem can be treated as finding the root such that:
\[\Ebb \big[1[R(a,x) - \lambda \geq 0] + e^{2\alpha} 1[R(a,x) - \lambda < 0] \,\big|\, A = a, X = x \big]=0.\] 

Note that $-q(\lambda)$ is the derivative of 
$Q(\lambda) := \frac{1}{2}\big[(R(a,x) - \lambda)^2_{+} +  e^{2\alpha}(R(a,x) - \lambda)^2_{-}  \big],$
so the intuition is that finding the root of $q(\lambda)$ is equivalent to finding the minimizer of $Q(\lambda)$. 

To be rigorous, by invoking the \emph{integrand theory} \cite{rockafellar2009variational}, which enables switching the integral operator with infimum under mild regularity conditions (the details of which we omit here to avoid unnecessary complications), the original problem is solved by:
\begin{equation}
    \underset{f(\cdot)}{\text{minimize}}\, \Ebb \big[ \big(R(a,x) - f(x)\big)^2_{+} +  e^{2\alpha}\big(R(a,x) - f(x)\big)^2_{-} \,\big|\, A=a, X=x \big],
\end{equation}
which gives the result in Lemma 2.
\end{proof}

\section{Proof of Theorem 1}\label{sec:proof_thm_main}
We first present a technical lemma.
\begin{lemma}[Ledoux-Talagrand contraction \cite{ledoux2013probability}]
\label{lemma:appen}
Let $f: \mathbb{R}_{+} \to \mathbb{R}_{+}$ be convex and increasing. Let $\phi_i : \mathbb{R}\to \mathbb{R}$ satisfy $\phi_i(0)=0$ and be Lipschitz with constant $L$, then for independent Rademacher random variables $\epsilon_i$, the following inequality holds for any $T \subset \mathbb{R}^n$:
\begin{equation}
\label{eqn:ledoux}
    \Ebb f \Big(\frac{1}{2} \sum_{t \in T} \Big| \sum_{i=1}^n \epsilon_i \phi_i(t_i) \Big| \Big) \leq \Ebb f \Big(\L \sum_{t \in T} \Big| \sum_{i=1}^n \epsilon_i t_i \Big| \Big).
\end{equation}
\end{lemma}

\begin{proof}
For our proof, we let 
\begin{equation}
\label{eqn:appen-def-S}
    S(\pi) = \sup_{\pi \in \mathcal{F}} \Big|\frac{1}{n}\sum_{i=1}^n \big( \sum_{a \in \mathcal{A}} \big\{ \pi(a|x_i) \hat{r}(a, x_i) + \Ibf(a=a_i) (r_i - \hat{r}(a, x_i)) \frac{\pi(a|x_i)}{\pi_0(a|x_i)}   \big\} \big) - V(\pi) \Big|,
\end{equation}
where we use an alternative expression for the DR (which is easy to verify). Define $w_0(a,x_i) = 1 / \pi_0(a|x_i)$. Then it is straightforward to show that for each $i$ in (\ref{eqn:appen-def-S}), we have:
\begin{equation*}
\begin{split}
    & \sum_{a \in \mathcal{A}} \big\{ \pi(a|x_i) \hat{r}(a, x_i) + \Ibf(a=a_i) (r_i - \hat{r}(a, x_i)) \frac{\pi(a|x_i)}{\pi_0(a|x_i)}   \big\}  \\ 
    & = \sum_{a \in \mathcal{A}} \big( \pi(a|x_i) \hat{r}(a, x_i) \big) + (r_i - \hat{r}(a_i, x_i)) \pi(a_i|x_i) w_0(a_i,x_i) \\
    & \in \big[-\frac{(q+1)M_{\alpha}+\overline{r}}{q}, \frac{(q+1)M_{\alpha}+\overline{r}}{q} \big].
\end{split}
\end{equation*}
Therefore, (\ref{eqn:appen-def-S}) has bounded difference with constant $2\frac{(q+1)M_{\alpha}+\overline{r}}{q}$. By the Hoeffding's inequality for bounded random variables, 
\begin{equation}
\label{eqn:appen-concentration-1}
    p\Big(S - \Ebb[S] \geq \delta \Big) \leq \exp \Big( \frac{- \delta^2 n q^2}{8((q+1)M_{\alpha} + \overline{r})^2}  \Big), \, \forall \delta > 0.
\end{equation}

Then we define the shorthand notation
\[ 
Q_i(a, x_i) =  \pi(a|x_i) \hat{r}(a, x_i) + \Ibf(a=a_i) \big(r_i - \hat{r}(a, x_i)\big) \frac{\pi(a|x_i)}{\pi_0(a|x_i)} .
\]
Notice that under $\pi_0$, DR is unbiased even if the RM estimator is misspecified, so $\Ebb \big[ \frac{1}{n}\sum_{i=1}^n \sum_{a\in\mathcal{A}} Q_i(a,x_i) \big] = V(\pi)$. Hence, let $x_i^{'}$ be an i.i.d copy of $x_i$ for all $i=1,\ldots,n$ and $\epsilon_{i,a}$ be the i.i.d Rademacher random variables, and we have
\begin{equation}
\label{eqn:appen-symmetrization}
\begin{split}
    \Ebb [S] & = \Ebb \Big[ \sup_{\pi \in \mathcal{F}} \Big| \frac{1}{n}\sum_{i=1}^n \sum_{a\in\mathcal{A}} Q_i(a, x_i) - \frac{1}{n}\sum_{i=1}^n \sum_{a\in \mathcal{A}} Q_i(a, x_i^{'}) \Big|   \Big] \\ 
    & = \Ebb \Big[ \sum_{\pi \in \mathcal{F}} \Big| \frac{1}{n}\sum_{i=1}^n \sum_{a\in \mathcal{A}} \epsilon_{i,a} \big(Q_i(a,x_i) - Q_i(a, x_i^{'}) \big) \Big|\Big] \\ 
    & \leq 2 \Ebb \sup_{\pi_\mathcal{F}} \Big| \frac{1}{n}\sum_{i=1}^n \sum_{a\in \mathcal{A}} \epsilon_{i,a} Q_i(a, x_i)   \Big|.
\end{split}
\end{equation}
It is straightforward to show that 
\[ 
\min \big\{ -M_{\alpha}, -\frac{-\overline{r}}{q} \big\} \leq Q_i(a,x_i) \leq \max \big\{ M_{\alpha}, -\frac{\overline{r}}{q} \big\}, 
\]
so by the Lemma \ref{lemma:appen}, we have:
\begin{equation}
\label{eqn:appen-talagrand}
    \Ebb \sup_{\pi_\mathcal{F}} \Big| \frac{1}{n}\sum_{i=1}^n \sum_{a\in\mathcal{A}} \epsilon_{i,a} Q_i(a, x_i) \Big| \leq 2\max \big\{ M_{\alpha}, -\frac{\overline{r}}{q} \big\} \Ebb \underbrace{\sup_{\pi \in \mathcal{F}} \big| \frac{1}{n}\sum_{i=1}^n \sum_{a\in \mathcal{A}} \epsilon_{i,a}\pi(a|x_i)  \big|}_{R_n(\mathcal{F})}.
\end{equation}

Notice that for the Rademacher complexity $R_n(\mathcal{F})$, each $\epsilon_{i,a} \pi(a|x_i)$ term is bounded by $[-1,1]$, so again by the bounded difference inequality, we can first obtain: 
\begin{equation}
\label{eqn:appen-concentration-2}
    p\Big( \Ebb R_n(\mathcal{F}) - R_n(\mathcal{F}) \geq \delta  \Big) \leq \exp \big( \frac{-\delta^2 n}{2} \big), \, \forall \delta > 0.
\end{equation}
Then, when $\pi = \pi_0$, we have:
\[  
\hat{V}(\pi_0) = \frac{1}{n}\sum_{i=1}^n\Big( \sum_{a\in\mathcal{A}} \big\{ \pi_0(a | x_i) \hat{r}(a,x_i) + (r_i - \hat{r}(a,x_i)\Ibf(a=a_i)  \big\}  \Big),
\]
and it is easy to show that each term inside the parentheses is bounded by \\ 
$ \big[ -\frac{(q+1)M_{\alpha} + q\overline{r}}{q},  \frac{(q+1)M_{\alpha} + q\overline{r}}{q} \big]$, so 
\begin{equation}
\label{eqn:appen-concentration-3}
    p\Big( \big| \hat{V}_{\text{DR}}(\pi_0) - V(\pi_0)  \big| \geq \delta  \Big) \leq \exp \Big( \frac{-\delta^2 n q^2}{8((q+1)M_{\alpha} + q\overline{r})^2} \Big), \, \forall \delta > 0.
\end{equation}

Combining (\ref{eqn:appen-concentration-1}), (\ref{eqn:appen-symmetrization}), (\ref{eqn:appen-talagrand}), (\ref{eqn:appen-concentration-2}) and (\ref{eqn:appen-concentration-3}), we have that for $\forall \delta_1, \delta_2, \delta_3 > 0$, and for $\forall \pi \in \mathcal{F}$,
\begin{equation}
\label{eqn:appen-concentration-4}
\begin{split}
    & p\Big( V(\pi) - V(\pi_0) \geq \hat{V}_{\text{DR}}(\pi) - \hat{V}_{\text{DR}}(\pi_0) -\delta_3 - \delta_2 - \delta_1 - 2\max\big\{ M_{\alpha}, \frac{\overline{r}}{q} \big\} R_n(\mathcal{F}) \Big) \\ 
    & \geq 1 - \exp \Big(\frac{- \delta_1^2 n q^2}{8((q+1)M_{\alpha} + \overline{r})^2}\Big) - \exp \Big( \frac{-\delta_2^2 n}{8 \big( \max \{M_{\alpha}^2, \frac{\overline{r}^2}{q^2} \} \big)} \Big) \\
    & \qquad \qquad \qquad \qquad \qquad  - \exp \Big( \frac{-\delta_1^2 n q^2}{8((q+1)M_{\alpha} + q\overline{r})^2} \Big)
\end{split}
\end{equation}
Finally, it is straightforward to verify that when $0 < q < \frac{1}{2}$, by setting $\delta_1 = \delta_2 = \delta_3$, the RHS of (\ref{eqn:appen-concentration-4}) is bounded below by $1-3\exp \Big(\frac{- \delta_1^2 n q^2}{8((q+1)M_{\alpha} + q\overline{r})^2}\Big)$. 

It is straightforward to see that the relation in the LHS of (\ref{eqn:appen-concentration-4}) still hold if we replace $\hat{V}_{\text{DR}}(\pi)$ by its lower bound. Hence, for any $\epsilon > 0$, we let $\delta_1 = 2\big( \frac{q+1}{q}M_{\alpha} + \overline{r}\big) \sqrt{\frac{2\log (3/\epsilon)}{n}}$ and obtain the final result of Theorem 1.
\end{proof}

\section{Explanation of Minorize-maximization procedure for DR}
In the reward-maximization setting where the candidate policy is parameterized as $\pi_{\theta} \in \mathcal{F} := \{\pi_{\theta} \, |\, \theta \in \Theta\}$, the minimax training objective becomes:
\begin{equation*}
    \underset{\pi_{\theta} \in \mathcal{F}}{\text{maximize}} \min_{\pi_U: e^{-\alpha} \leq \pi_0/\pi_{U} \leq e^{\alpha}} \hat{V}(\pi_{\theta}; \pi_U),
\end{equation*}
Notice that the above objective function can be nonconcave-nonconvex in general, so the Minimax Theorem \citep{sion1958general} does not apply, and the order of $\min$ and $\max$ can not be switched. Therefore, the lower-bounding methods from the previous section, which holds for \emph{any} given $\pi_{\theta}$, play an important role in finding the approximate solution for the minimax problem.
% has important implications. By our design, for each candidate policy $\pi_{\theta}$, the uncertainty in $\pi_U$ behaves adversarially to undermine the off-policy evaluation result for $\pi_{\theta}$. As a consequence, it enforces the robustness of $\pi_{\theta}$ on making prediction under future uncertainty. 
We illustrate how to plug in the lower bounds using the DR estimator as an example, since it includes both the IPS and RM estimator as special cases. 
We use $\hat{V}_{DR}(\pi_{\theta};p,r)$ to denote the objective function for lower-bounding the DR estimator, where $p:=p(a|x)$ and $r:=r(a,x)$ are the optimization policy and reward model, and $\Pi(\pi_0;\alpha)$ and $\mathcal{R}(\pi_0;\alpha)$ give the corresponding constraint sets. It then holds that:
\begin{equation}
\label{eqn:minimax-DR}
\begin{split}
    & \max_{\pi_{\theta} \in \mathcal{F}}\, \min_{e^{-\alpha} \leq \pi_0/\pi_{U} \leq e^{\alpha}} \hat{V}_{\text{DR}}(\pi_{\theta}; \pi_U) \\
    & \geq \max_{\pi_{\theta} \in \mathcal{F}}\, \inf_{p \in \Pi(\pi_0;\alpha), r \in \mathcal{R}(\pi_0;\alpha)} \hat{V}_{DR}(\pi_{\theta};p,r) \\
    & \geq \max_{\pi_{\theta} \in \mathcal{F}} \underline{\hat{V}_{\text{DR}}(\pi_{\theta})}.
\end{split}
\end{equation}
Therefore, we can instead consider optimizing the lower bound of the minimax problem, which leads to a strict non-decreasing optimization path and can be implemented efficiently. In particular, we adopt the minorize-maximization approach described in Algorithm 1. We also provide a sketched visual illustration in Figure \ref{fig:optimization} to show the working mechanism. Notice that the constraint sets $\mathcal{R}(\pi_0;\alpha)$ and $\Pi(\pi_0;\alpha)$ do not depend on $\pi_{\theta}$, so we only need to compute them once at the beginning of the optimization. 

\begin{figure}[bht]
    \centering
    \includegraphics[width=0.6\linewidth]{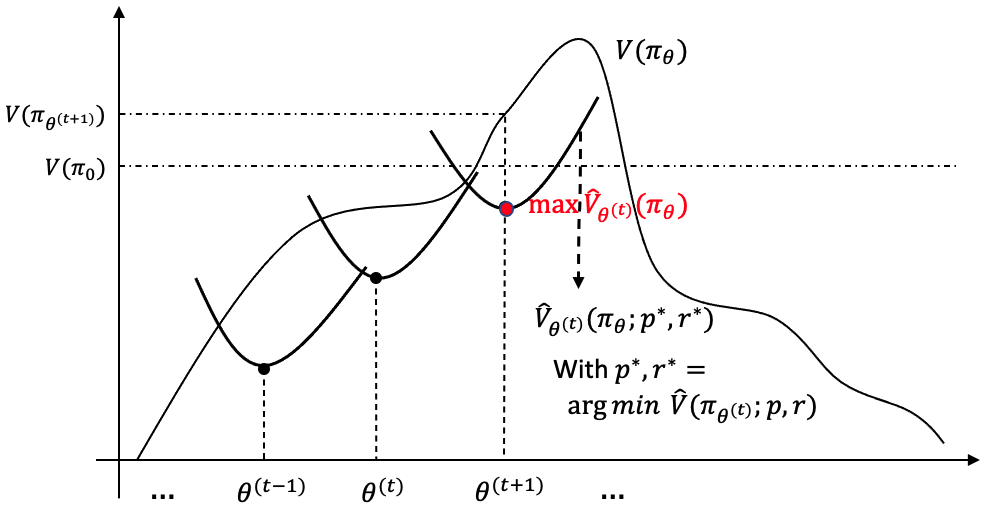}
    \caption{A sketched visual illustration for the minorize-maximization optimization procedure.}
    \label{fig:optimization}
\end{figure}

\section{Supplement Material for Experiments and Results}

We provide the detailed experiment descriptions and complete numerical results in this part of the paper.

\subsection{Simulation settings}

\textbf{Training, validation and testing.}

For all the benchmark datasets, we do the train-validation-test split by 56\%-24\%-20\%. We split the data this way because some of the datasets in the UCI repository\footnote{https://archive.ics.uci.edu/} already has a train-test split so we adopt the provided setting. 

We first point out that the purpose of our our train-validation-test split are two-folds:
\begin{itemize}
    \item a regular train-validation-test split for off-policy optimization;
    \item an extra train-validation split within the regular training data for computing the surrogate functions for RM (via the auxiliary ERM).
\end{itemize}

To obtain the upper and lower bound functional for the RM estimator, we conduct the auxiliary risk minimization procedure described in Section 3.1. We use the \emph{gradient boosting regression tree} \cite{chen2016xgboost} as the parametric prediction function. 
Now there is an extra via in training surrogate functions $\hat{f}$ and $\hat{g}$ using the auxiliary ERM, because their validation should not be carried out on the usual validation set to avoid information leak. To tune the hyper-parameters of the surrogate functions for the RM, we further split the training data into RM-specific training and validation set by 80\%-20\% (Figure \ref{fig:appen-split}). The original validation dataset is kept untouched in this step.

For training gradient boosting regression tree, we choose the \emph{learning rate} from $\{0.01, 0.1, 0.3, 0.5\}$ and the \emph{maximum depth} from $\{2,3,4\}$ according to the risk on the RM-specific validation set. The other parameters are kept as default in Xgboost.

\begin{figure}[hbt]
    \centering
    \includegraphics[width=0.5\linewidth]{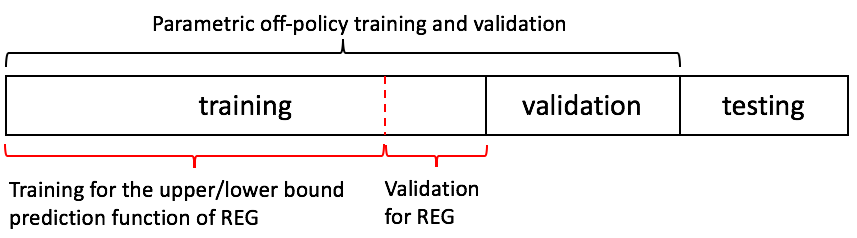}
    \caption{Training-validation-testing split}
    \label{fig:appen-split}
\end{figure}

The parametric off-policy optimization for $\pi_{\theta}$ is then conducted on the standard training and validation set. 

% As we mentioned in Section 5.1, we have access to the oracle estimator under the cost-sensitive classification setting. Therefore, we can either use the oracle estimator's value or the DR objective value in Algorithm 1 as the validation metric. In our experiments, we find them giving almost the same model choices for all datasets under all the $\alpha$ we considered.  

\textbf{{The multi-layer perceptron as parametric policy class}}

We consider the multi-layer perceptron (MLP) as our parametric policy class. Since the number of features is quite small for the benchmark datasets, we decide to use the \emph{two-layer} MLP with ReLU as activation function, which is capable of capturing non-linear feature interaction effects. 

We treat the number of hidden factors in the intermediate layer as the tuning parameter and select it from $\{3,5,7,9,11\}$. We do not consider dropout or regularization since the number of features is small in the simulations. We use the Adam optimizer with fixed a learning rate.

\textbf{{Computation}}

We leverage the open-source package \emph{Xgboost}\footnote{https://github.com/dmlc/xgboost} for training the {gradient boosting regression trees} to obtain the surrogate functions of RM. We choose \emph{Xgboost} for the specific reason that it allows us to customize the loss function, so we can directly work with the auxiliary ERM. 

The two-layer MLP for our off-policy optimization is implemented with the auto-differentiation framework \texttt{PyTorch}\footnote{https://pytorch.org/}. All the models are trained on a 16-Gb Linux-system machine with one Tesla V100 GPU.

As for computing the lower bounds for IPS and DR according to Equation (13) and (14), we apply off-the-shelf solver since the explicit expressions for the updates are explicit. 

All the datasets and executable implementation code are provided as a part of the supplementary material.

\subsection{Additional numerical results for simulation}

We discuss how the DR objective value, given by the $\hat{V}_{\text{DR}}(\pi_{\theta^{\text{new}}})$ in Algorithm 1, as well as the \emph{oracle} value of the candidate policy changes on the testing data as the training progresses.

\textbf{{The oracle value}}

Recall from the simulation setting that we have access to the following \emph{oracle} value that reveals the expected reward of the candidate policy:
\[\Ebb \Big [\sum_{a \in \mathcal{A}}r(a) \pi_{\theta}(a|x) \Big],\] 
where the expectation is taken with respect to the empirical data distribution.
This is because the feedback data is constructed from the cost-sensitive classification setting, so we have access to the full set of rewards that enables the computation of the oracle value.

\textbf{{The impact of $\alpha$ on training}}

For each benchmark dataset, we show the training progress of DR objective value ($\hat{V}_{\text{DR}}(\pi_{\theta})$ in Algorithm 1) evaluated on the testing data, in the first column of Figure \ref{fig:appen-complete}. 

In the second to the fourth column, we compare the DR objective values and the oracle values during the training process, also on the testing data, under different $\alpha$. 

\begin{figure}
    \centering
    \includegraphics[width=0.77\linewidth]{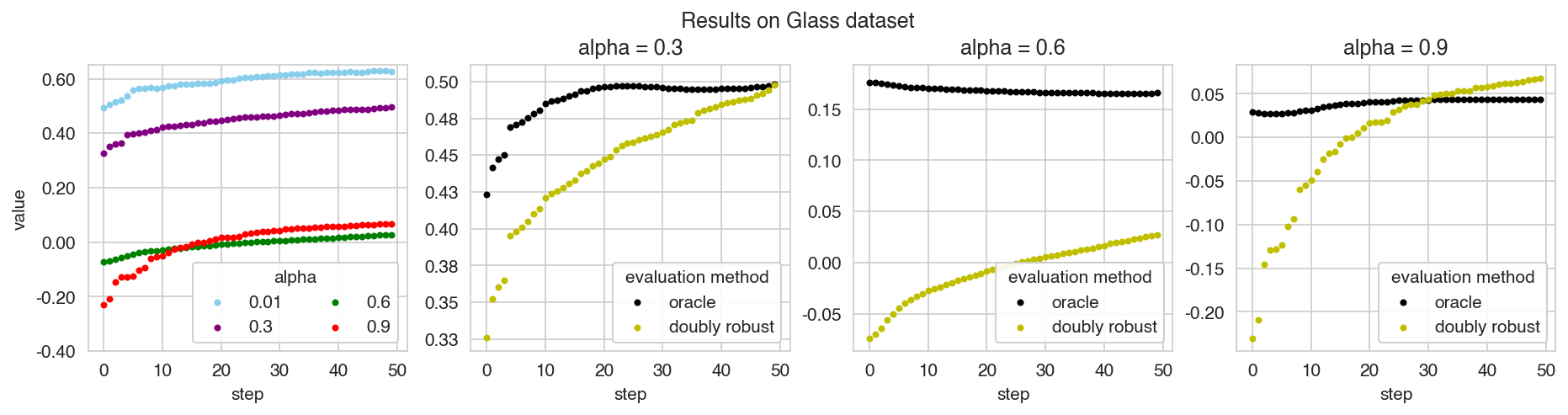}
    \includegraphics[width=0.77\linewidth]{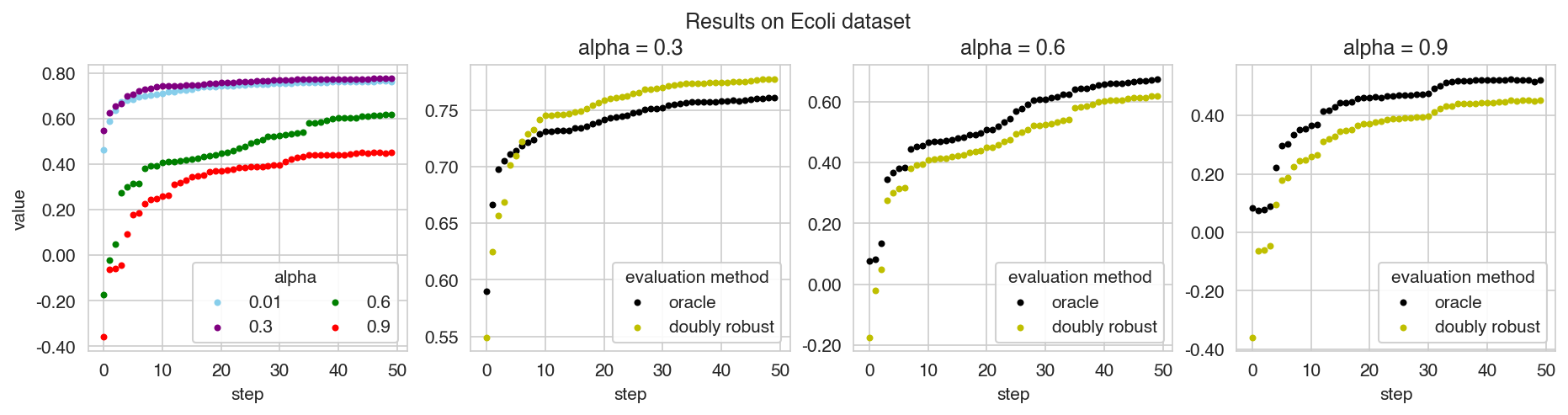}
    \includegraphics[width=0.77\linewidth]{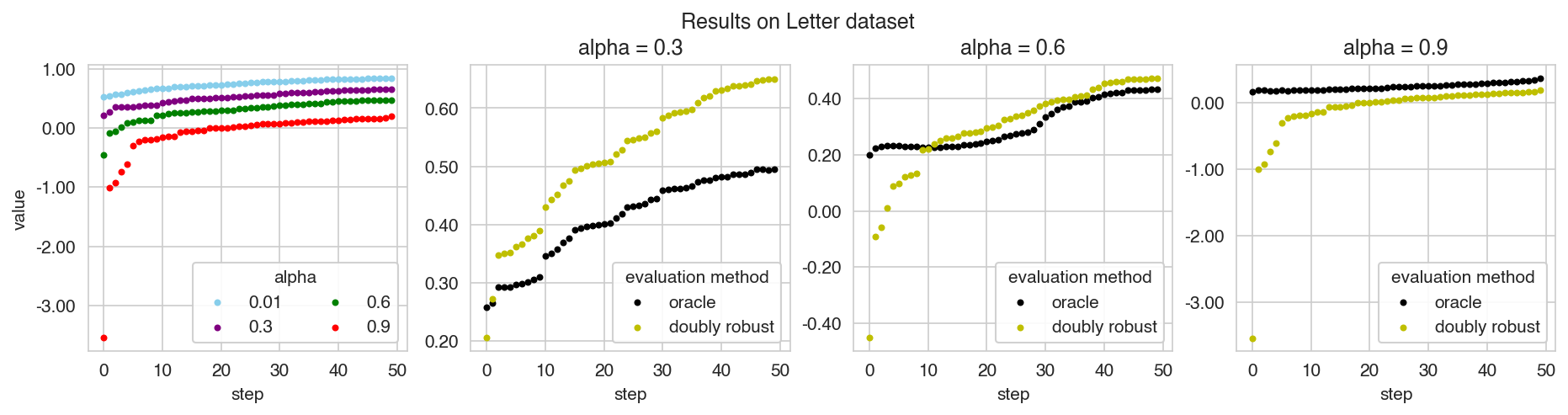}
    \includegraphics[width=0.77\linewidth]{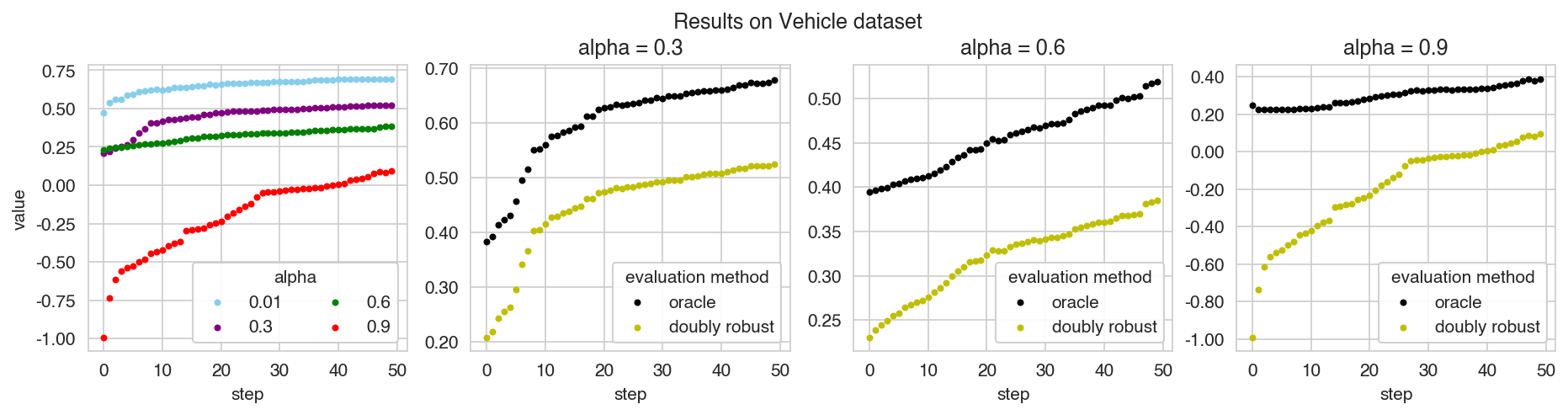}
    \includegraphics[width=0.77\linewidth]{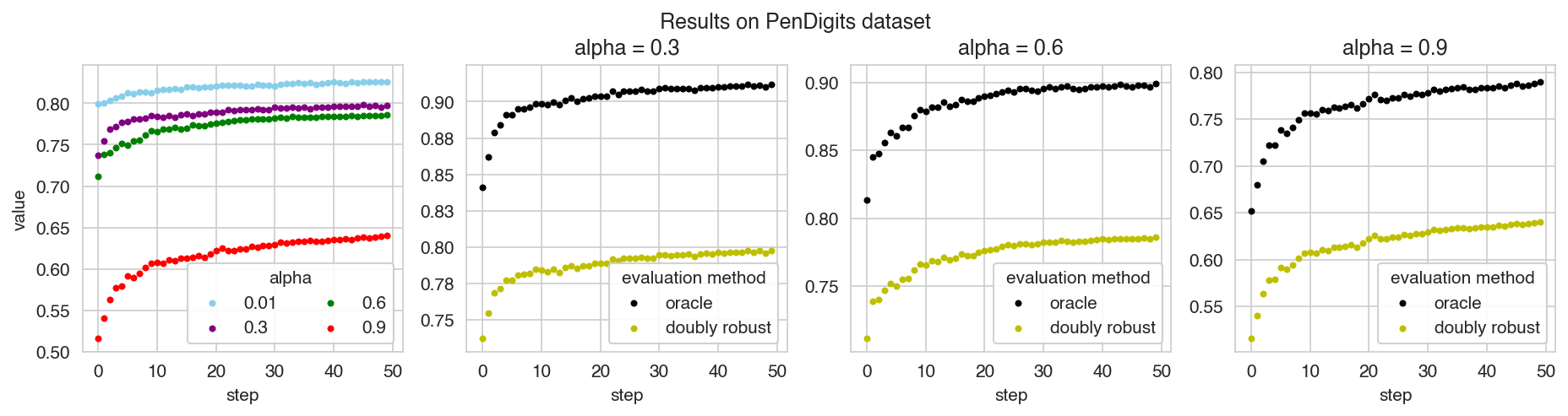}
    \includegraphics[width=0.77\linewidth]{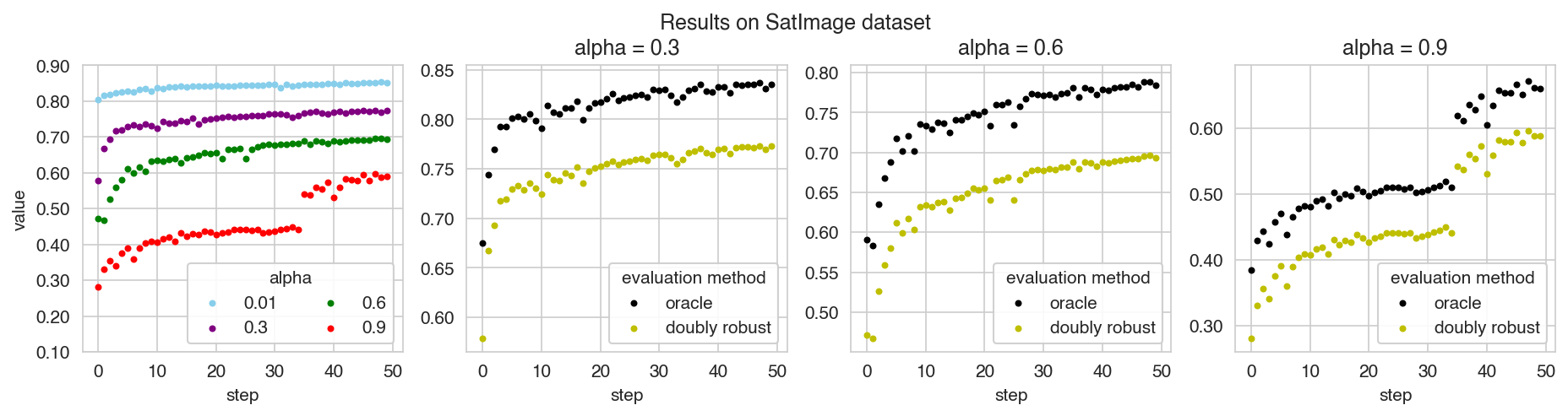}
    \caption{The training progress for the proposed off-policy optimization under different $\alpha$.}
    \label{fig:appen-complete}
\end{figure}

We make the following conclusions from the numerical results.
\begin{itemize}
    \item Firstly, the observation that a larger $\alpha$ disturbs the off-policy optimization is consistent across all datasets. When the training stabilizes, the DR objective value under larger $\alpha$ is uniformly smaller, which is obvious by looking at the first column of Figure \ref{fig:appen-complete}. 
    
    \item Secondly, when $\alpha$ is moderate, e.g. $\alpha=0.3$ (the second column of Figure \ref{fig:appen-complete}), the oracle value is improving consistently during the training process, which demonstrates the effectiveness of the proposed minimization-maximization optimization. 
    
    \item Thirdly, when $\alpha$ gets very large, making improvements in terms of the oracle value becomes impossible. Even though the DR objective value is still increasing on the testing data, the oracle value remains unchanged (in particular, the \texttt{Glass} dataset with $\alpha=0.6$ and $\alpha=0.9$, the \texttt{Letter} dataset with $\alpha=0.9$ in Figure \ref{fig:appen-complete}).
    
    \item Finally, the tradeoff between the degree of uncertainty and optimization accuracy (effectiveness) is different for each dataset. For instance, the tradeoff on the \texttt{Glass} dataset take place for $\alpha \leq 0.6$ (the first row in Figure \ref{fig:appen-complete}), but on the \texttt{Ecoli} and \texttt{PenDigits} datasets, we may expect $\alpha \geq 0.9$ (the second and the fifth row of Figure \ref{fig:appen-complete}).
\end{itemize}

In a nutshell, the proposed minimization-maximization algorithm effectively handles off-policy optimization under moderate $\alpha$. When $\alpha$ get too large, we bring too much fluctuation to the training, which may eventually cause an overlarge gap between the DR maximization objective (the purple, blue and green curves in Figure 2) and the true value (the red curve in Figure 2), making it impractical to achieve further improvements. 

The above numerical results further support our theoretical discussions following Theorem 1. Here, we show from a practical perspective on how the different choices of $\alpha$ could affect the training progress. Notice that the impact of $\alpha$ depends on the dataset, so it intrinsically reveals a data-specific property in terms of making robust prediction. Therefore, it is always recommended to apply the prior knowledge or carrying out pilot experiments to determine the suitable $\alpha$.

\subsection{Real-world experiments}
\label{sec:real-world}

The motivation of our work stems from the demand of making robust prediction during off-policy learning for real-world problems. For our experiments, we design off-policy learning for one of the largest e-commerce platforms in the U.S, using both the history data and the online production environment.

\textbf{Problem setting and experiment descriptions}

The daily personalized homepage recommendation of the platform relies on the contextual bandit, and the feedback data is generated according to the description in Section 3 of the paper. 

The (frequent) customers of the platform are tagged with \textit{personas} such as pet lover, sport enthusiast, busy family, etc. The actions $a$ corresponds to the candidate products, which is of the magnitude of several hundreds designated by the business team. Each candidate product has an \emph{affinity score} with the customer personas, e.g. $x_{\text{product, persona}}$. The context is given by concatenating all the persona-product affinity scores, i.e. $\vec{x} = \big(x_{1,1}, \ldots, x_{n_1, k} \big)$, where $n_1$, $k$ are the number of personas and number of products (actions). The data-collection workflow can be described as below:
\begin{itemize}
    \item the front-end sends an request when a customer lands on the homepage;
    \item the feature vector $\vec{x}$ is acquired from the megacache according to the personas of the user. Here, $x_{i,j}, j=1,\ldots,k$ is treated as missing value if the customer is not tagged with feature $k$;
    \item the top-ten recommendation products are sampled without replacement according to $\pi_{\text{log}}(a|\vec{x})$.
\end{itemize}
On our e-commerce platform, the personas are updated on a daily basis. This is equivalent to say that the policy value is unchanged for each customer throughout the day. So in practice, the policy values $\pi_{\text{log}}(a|\vec{x})$ can be stored in the cache after the first-time computation, and need not be re-computed until the personas are updated. Also, even though the top-ten recommendations are made mutually exclusive (the same item cannot appear twice), we still treat them as being sampled independently.

Here, the $\pi_{\text{log}}$ is finalized according to the following process:
\begin{itemize}
    \item the data science team proposes a design policy $\pi_{\text{design}}$ according to the off-policy learning, and push the model to the mega cache for online serving;
    \item the business team decides how to manage the traffic at peak time and make adjustments under real-time supply-demand availability, which leads to the logging policy $\pi_{\text{log}}$.
\end{itemize}
We provide a sketched diagram in Figure \ref{fig:diagram} to illustrate on the workflow. 

\begin{figure}[htb]
    \centering
    \includegraphics[width=0.6\linewidth]{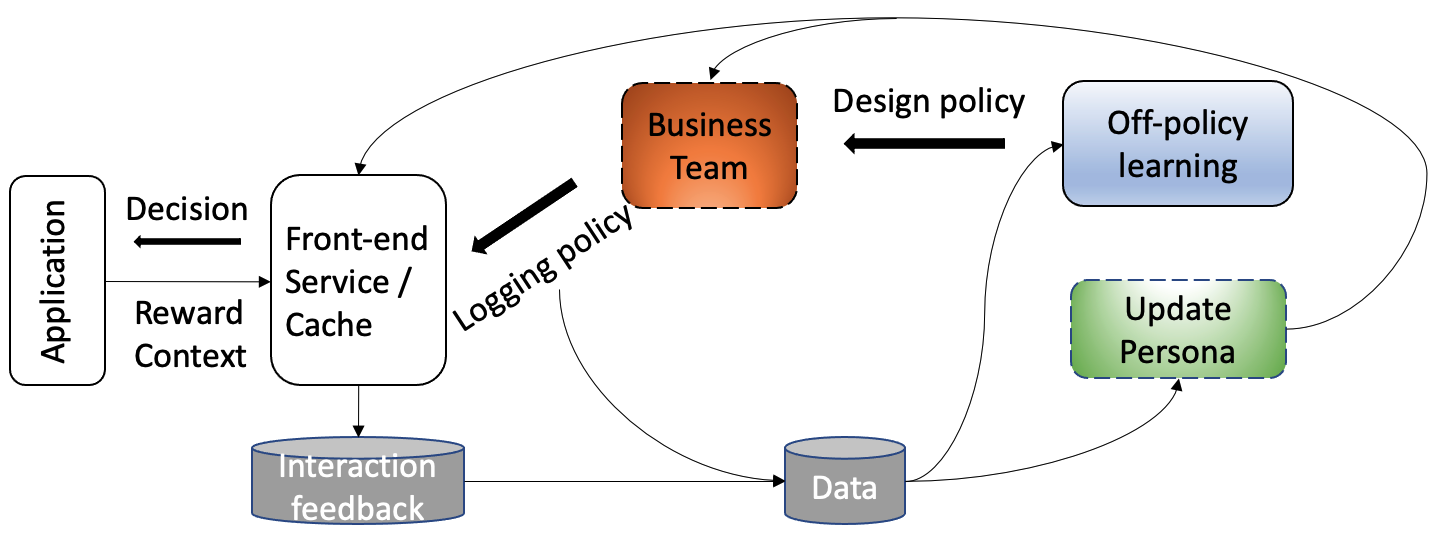}
    \caption{A sketched diagram of the workflow on the e-commerce platform. Here, we omitted the irrelevant parts and only show the components that are relevant to the off-policy learning.}
    \label{fig:diagram}
\end{figure}

For the offline experiments, we collect a feedback data that consists of $\sim$ 200 product (actions) and $\sim$200,000 samples. The history logging policy $\pi_0$ is revealed, where the extreme propensity score values do occur, e.g. $\pi_0(a|x) \leq 1e^{-4}$. By a preliminary experiment, we find using $q=0.01$ as the cutoff gives a reasonable bias-variance tradeoff. 

The full feedback data is split into training, validation and testing set by 70\%, 15\%, 15\%, according to the timestamp in a chronological order. We use the \emph{gradient boosting regression tree} as parametric policy, and apply the minimization-maximization procedure in Algorithm 1 for training. The turning parameters are selected according to the objective value $\hat{V}(\pi_{\theta})$ from Algorithm 1 on the validation data. In particular, we treat the learning rate, number of estimators and maximum tree depth as tuning parameters.

\textbf{Complete numerical results of real-world experiments}

We experiment with $\alpha \in \{0.01, 0.05, 0.1, 0.3\}$ for off-policy learning on the real-world dataset. The outcome is provided in Figure \ref{fig:real-dat-alpha}, where we compare the \textbf{standard} off-policy learning and the proposed \textbf{robust} off-policy learning approach. 

\begin{figure}[htb]
    \centering
    \includegraphics[width=\linewidth]{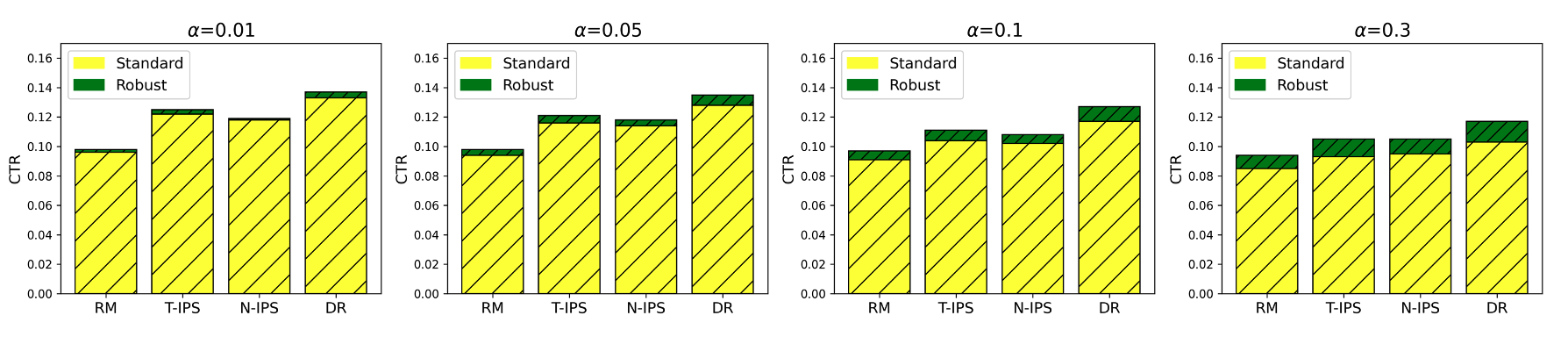}
    \caption{The real-world off-policy learning outcome (examined on the testing data) under different uncertainty degree $\alpha$.}
    \label{fig:real-dat-alpha}
\end{figure}

As we mentioned before, the IPS method suffers from stability issue, so we do not consider it here. For the truncated IPS method, we use $q=0.005$ as the cutoff on the propensity score; and for the normalized IPS method, we use the workaround proposed in Appendix \ref{sec:variant-ips} to conduct the proposed off-policy learning. For the DR method, we also use $q=0.01$ to truncate the propensity score to avoid large variances. The standard approach and the proposed approach are trained and validated on the same datasets using the same settings.

From the results in Figure \ref{fig:real-dat-alpha}, we observe a similar pattern as in the simulation where the proposed prediction-robust approach provide more significant improvements under larger $\alpha$. When $\alpha$ is relatively small, the performance is almost the same as the standard approach. It demonstrates the effectiveness of the proposed approach using the actual feedback data (rather than generated) when uncertainty exists during the testing time. Also, the DR method gives the best outcome uniformly, while the nIPS method is still outperformed by the tIPS, probably due to the inefficiencies of the workaround approach that we analyzed in Section \ref{sec:variant-ips}.

% \textbf{Online A/B testing}

% We carry out online A/B testing to accurately compare the policies optimized by the standard off-policy learning and the proposed approach. Both candidates use the same history feedback data for training. During the online deployment, the business team works with the two candidates separately (without knowing which is the treatment and which is the control), and make the adjustments in the face of different online scenarios. We coordinate with the business team on the proper choice of $\alpha$ by the empirical experience. The CTR results are aggregated on a daily basis (because the personas are updated each day), we report the mean and variance according to the daily-aggregated outcome during a 21-day testing period.

%\bibliographystyle{apalike}
%\bibliography{references}

%\vfill

\end{document}